\documentclass[10pt,twocolumn]{IEEEtran}
\usepackage{amsmath,amssymb,euscript,yfonts,psfrag,latexsym,graphicx}
\usepackage{bbm,color,amstext,wasysym,parskip}
\usepackage{amsthm}
\graphicspath{{./},{../figures/},{./figures/}}
\usepackage{caption}
\usepackage{subcaption}
\usepackage{float}
\usepackage{array}
\usepackage{url}
\usepackage{algorithm}
\usepackage{algorithmic}
\usepackage{xcolor}

\usepackage[normalem]{ulem}

\usepackage{hyperref}

\graphicspath{{./},{../figures/},{./figures/}}

\newtheorem{thm}{Theorem}

\newtheorem{rem}{Remark}
\newtheorem{cor}{Corollary}
\newtheorem{prob}{Problem}

\usepackage{soul}

\newcommand{\bo}{{\mathbf o}}

\newcommand{\bn}{{\mathbf n}}
\newcommand{\by}{{\mathbf y}}

\newcommand{\mN}{{\mathbf N}}

\newcommand{\bK}{{\mathbb K}}
\newcommand{\bN}{{\mathbb N}}

\newcommand{\bB}{{\mathbf B}}

\newcommand{\mR}{{\mathbb R}}

\newcommand{\cL}{{\mathcal L}}

\newcommand{\bx}{{\bf x}}
\def\bC{{\bf C}}
\def\bK{{\bf K}}

\def\bU{{\bf U}}

\def\minwrt[#1]{\underset{#1}{\text{minimize}}}
\def\maxwrt[#1]{\underset{#1}{\text{maximize }}}

\begin{document}
\title{Inference with Aggregate Data: An Optimal Transport Approach}

\author{Rahul Singh, Isabel Haasler, Qinsheng Zhang, Johan Karlsson, and Yongxin Chen
\thanks{This work was supported by the Swedish Research Council (VR), grant 2014-5870, KTH Digital Futures, and the NSF under grant 1901599 and 1942523.
}
\thanks{R. Singh, Q. Zhang and Y. Chen are with the School of Aerospace Engineering,
Georgia Institute of Technology, Atlanta, GA, USA. {\tt\small \{qzhang419,rasingh,yongchen\}@gatech.edu}}
\thanks{I.~Haasler and J.~Karlsson are with the Division of Optimization and Systems Theory, Department of Mathematics, KTH Royal Institute of Technology, Stockholm, Sweden. {\tt\small haasler@kth.se, johan.karlsson@math.kth.se}}
\thanks{R. Singh, I.~Haasler, and Q. Zhang contribute equally to this paper.}}

\maketitle

\begin{abstract}
We consider inference (filtering) problems over probabilistic graphical models with aggregate data generated by a large population of individuals. We propose a new efficient belief propagation type algorithm over tree-structured graphs with polynomial computational complexity as well as a global convergence guarantee. This is in contrast to previous methods that either exhibit prohibitive complexity as the population grows or do not guarantee convergence. Our method is based on optimal transport, or more specifically, multi-marginal optimal transport theory. 
In particular, we consider an inference problem with aggregate observations, that can be seen as a structured multi-marginal optimal transport problem where the cost function decomposes according to the underlying graph.
Consequently, the celebrated Sinkhorn/iterative scaling algorithm for multi-marginal optimal transport can be leveraged together with the standard belief propagation algorithm to establish an efficient inference scheme which we call Sinkhorn belief propagation (SBP).
We further specialize the SBP algorithm to cases associated with hidden Markov models due to their significance in control and estimation.
We demonstrate the performance of our algorithm on applications such as inferring population flow from aggregate observations. We also show that in the special case where the observations are generated by a single individual, our algorithm naturally reduces to the standard belief propagation algorithm.
\end{abstract}

\section{Introduction}\label{sec:intro}
Filtering problems can be more generally posed as inference problems in probabilistic graphical models (PGMs) \cite{WaiJor08} in a large number of applications including robot localization and mapping, object tracking, and control \cite{Thr02,AndMoo12}.
PGMs provide a powerful framework for modeling the dependence and relations between probabilistic quantities and probabilistic inference using PGMs have been widely used in signal processing, computer vision, computational biology, and many other real-world applications~\cite{WaiJor08,KolFri09,BoxTia11}. During the last decades, many inference algorithms have been proposed, among which the belief propagation (BP) algorithms~\cite{Pea88,YedFreWei01,YedFreWei03} have been extremely effective and successful. The standard PGM framework and the associated inference algorithms are suitable for modeling of distinguishable/labeled individuals as in most standard applications. In fact, many standard filtering algorithms such as Kalman filter are instances of the BP algorithm operating in a special graphical model known as hidden Markov model (HMM)~\cite{KscFreLoe01}. The inference problem in HMMs is a filtering problem that estimates dynamically changing unobservable (hidden) states from a sequence of noisy observed data.

Recently, there has been growing interest in applications involving a large population of individuals, e.g., in animal migration and crowd estimation, where data about individuals are not available. Instead, aggregate population-level observation in the form of counts or contingency tables are provided~\cite{Sun75,Mac77,KalLawVol83,SheDie11,LuoXuZhe16}. A distinct feature of this setting is that the individuals are no longer distinguishable to each other. This restriction in the observations could be due to privacy, security, or economic reasons.
For example, in tourist flow analysis, individual trajectories may not be readily accessible, but the number of people in a given area can typically be counted using video surveillance or electronic gates. Similarly, it is much easier to obtain the population sizes of bird migration than tracking the trajectory of each bird. Finally, when studying the spreading of COVID-19, often only aggregate data is available due to privacy issue; normally patient data needs to be anonymized. In this setting with aggregate observations, the traditional inference algorithms such as BP in PGMs which heavily depend on individual observations, are thus not directly applicable.
The development of reliable and efficient aggregate inference algorithms is of great importance and necessity.

The collective graphical model (CGM) introduced by \cite{SheDie11} is a recent framework for  filtering\footnote{We use the terms ``filtering" and ``inference" interchangeably in the paper.} (inference) and learning with aggregate data. 
A CGM is a graphical model that describes the histograms of individuals directly. Using this model, it was proved that the complexity of traditional inference algorithms for exact inference scales at least polynomially as the population grows \cite{SheSunKumDie13}. To circumvent this difficulty, they proposed an approximate maximum a posteriori (MAP) formulation, which approximates the marginal inference solution well, especially when the size of the population is large \cite{SheSunKumDie13}. Under some proper assumptions, the approximate MAP formulation is a convex optimization problem and the problem dimension is independent of the population size. 
To further accelerate the inference, they proposed the non-linear belief propagation (NLBP) algorithm \cite{SunSheKum15}.
The NLBP algorithm is a message-passing type algorithm relying on interactions between graph nodes. 
Despite of its similarity to the BP algorithm \cite{Pea88}, NLBP suffers from instability and lack of convergence. Indeed, no convergence guarantee has been established so far \cite{SunSheKum15}. Bethe-RDA is another existing algorithm for aggregate inference based on regularized dual averaging (RDA) \cite{VilBelSheMcc15}. It is a type of proximal gradient descent algorithm that exhibits convergence guarantee.

The goal of this paper is to establish an efficient and reliable algorithm for filtering with aggregate observations described by the precise distribution of the observation variables. This observation model is different from the one used for NLBP and Bethe-RDA and neither of these two algorithms works for our problems; see Figure \ref{fig:noise_model_b} and the associated explanations in Section \ref{sec:main_result} for more discussions on the observation models. We build on the CGM framework and develop such an aggregate inference algorithm.  
Our algorithm is based on multi-marginal optimal transport  (MOT) theory \cite{GanSwi98,Pas15,BenCarCut15,Nen16,Pas12}, which involves the transport among multiple marginal distributions.
MOT is a generalization of the classical optimal transport (OT) problem \cite{Vil03} of Monge and Kantorovich to find a transport plan from a source distribution to a target one that minimizes the total transport cost. 
Within the MOT framework, the aggregate observations are viewed as fixed given marginal distributions. We show that the aggregate inference problem in CGM reduces to the special case of entropic regularized formulation of MOT with marginals specified by these aggregate observations. 
Thanks to this equivalence, the aggregate inference problem can be solved by the popular Sinkhorn a.k.a., iterative scaling algorithm \cite{Sin64,FraLor89,Cut13,BenCarCut15}. The Sinkhorn algorithm has the advantages of being extremely easy to implement and parallelize, and has global convergence guarantee~\cite{FraLor89,BenCarCut15}. We show that the Sinkhorn algorithm for aggregate inference can be further accelerated by leveraging the underlying graphical structure of the inference problems with aggregate observations; a key projection step in the Sinkhorn algorithm, which could be potentially expensive, can be realized efficiently by standard BP for tree-structured graphical models. This accelerated version of our algorithm is named Sinkhorn belief propagation (SBP). 

SBP exhibits convergence guarantees 
when the underlying graph is a tree. 
The contributions of the paper are summarized as follows.
\begin{itemize}
    \item We discover an equivalent relation between OT theory and inference/filtering problems with aggregate observations;
    \item Based on OT theory and belief propagation, we propose an efficient marginal inference/filtering algorithm with aggregate data that has a global convergence guarantee; 
    \item We study the filtering problem in collective HMMs and establish connections between collective and standard filtering problems;
    \item We demonstrate the performance of our algorithm on applications such as inferring population flow from aggregate observations.
\end{itemize}

\textbf{Related Work:}
Early works on aggregate data focused on the learning of the parameters of the underlying models. For example, \cite{Sun75,Mac77,KalLawVol83} studied the modeling of a single Markov chain by maximizing the aggregate posterior. More recent learning methods from aggregate data include~\cite{PasFuBou12,LuoXuZhe16}. Since the formalism of CGMs by~\cite{SheDie11} there have been multiple works on inference for aggregate data. The complexity of exact inference in CGMs has been investigated in~\cite{SheSunKumDie13} and an approximate MAP formulation has been proposed in the same paper. The non-linear belief propagation algorithm~\cite{SunSheKum15} is a message passing type algorithm for approximate MAP inference in CGMs, but it does not have a convergence guarantee. The learning of a Markov chain within the CGM framework has been presented in~\cite{BerShe16}. On the other hand, the application of OT theory in filtering and estimation problems have been investigated in \cite{CheKar18,HasRinChe19,CheGeoPav15a,TagMeh20}. Another closely related problem is the Schr\"odinger bridge problem \cite{Leo14,CheGeoPav14e,CheConGeoRip19}, which is essentially equivalent to an entropic OT problem. Our work is also closely related to linear/nonlinear filtering \cite{Kal60,Rab89,ItoXio00,AndMoo12,YanMehMey13,LorNvd11}, which is widely studied in the control community. To some degree, our framework can be viewed as an extension of standard filtering theory to the setting with aggregate observations. 

 The rest of the article is organized as follows. In Section~\ref{sec:background}, we briefly discuss related background including PGMs, BP, CGMs, and MOT. We present our main results and the Sinkhorn belief propagation algorithm (Algorithm \ref{alg:is_bp}) in Section~\ref{sec:main_result}. In Section \ref{sec:Filtering_hmm}, we specialize SBP to HMMs and obtain the collective forward-backward algorithm. It is followed by the experimental results in Section~\ref{sec:eval} and conclusion in Section~\ref{sec:conclusion}.

\section{Background}
\label{sec:background}
In this section, we briefly present relevant background on the components of our method which includes probabilistic graphical models, collective graphical models, and multi-marginal optimal transport. 
\subsection{Probabilistic Graphical Models}
\label{subsec:pgm}
Probabilistic graphical models \cite{WaiJor08} are graph-based representations of a collection of random vectors that capture the conditional dependencies between them. A PGM is associated with an underlying graph $G=(V,\,E)$ where $V$ and $E$ denotes the set of vertices and edges respectively. Each node $i\in V$ corresponds to a random variable $X_i$ which can be either discrete or continuous, though we consider the setting with discrete random variables throughout. The random variable at each node takes values from the same\footnote{This is only for the ease of notation. In general, these random variables can take values in different sets.}
finite set $\mathcal{X}$, with cardinality $|\mathcal{X}| = d$. Assuming that the underlying graph is undirected with $J=|V|$ nodes, the probability of the PGM is given by
	\begin{equation}\label{eq:MRF}
		p(\bx) = p(x_1,x_2,\ldots,x_J) = \frac{1}{Z} \prod_{(i,j) \in E} \psi_{ij}(x_i,x_j),
	\end{equation}
where $\bx=\{x_1,\ldots,x_J\}$ is a particular assignment to the corresponding random variables, $\psi_{ij}$ are known as edge potentials and $Z$ is a normalization constant.
The edge potential $\psi_{ij}$ describes the correlation between the random variables $X_i$ and $X_j$. For instance, the joint probability distribution of the four random variables in the PGM in Figure~\ref{fig:pgm_example} factorizes as
\begin{equation*}
    p(x_1,x_2,x_3,x_4) = \frac{1}{Z} \psi_{12}(x_1,x_2) \psi_{23}(x_2,x_3) \psi_{24}(x_2,x_4).
\end{equation*}
\begin{figure}[h]
\centering
\includegraphics[scale=1]{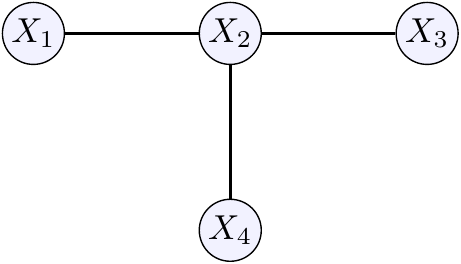}
\caption{An example PGM.}
\label{fig:pgm_example}
\end{figure}
Occasionally, (local) node potentials $\phi_i(x_i)$ induced by, e.g., measurements, are also included in \eqref{eq:MRF} to define the joint probability, but they can always be absorbed into the edge potentials \cite{WaiJor08}. Thus, for the sake of simplicity and without loss of generality, we adopt the probability structure \eqref{eq:MRF}.

Furthermore, we restrict our attention to PGMs with undirected underlying graphs. Indeed, a large class of PGMs are associated with directed graphs, including hidden Markov models (HMMs) which are widely used in the control and estimation community, however these directed models can be converted to undirected ones using a technique known as \textit{moralization}~\cite{KolFri09}. More specifically, the equivalent moralized graph of a directed graph is formed by converting all edges in the graph into undirected ones and adding additional edges between all pairs of non-adjacent nodes with a common child node. As an example, consider the directed graph in Figure~\ref{fig:moralization_a}. It has two non-neighboring nodes $X_1$ and $X_3$ with a common child node $X_2$ and thus we add an additional edge between nodes $X_1$ and $X_3$ in the undirected moralized counterpart depicted in Figure~\ref{fig:moralization_b}.

\begin{figure}[h]
\centering
\begin{subfigure}{.23\textwidth}
\centering
\includegraphics[scale=1]{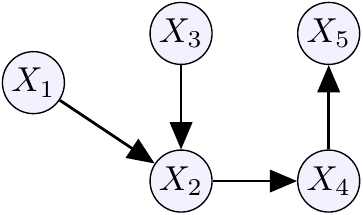}
\caption{Directed}
\label{fig:moralization_a}
\end{subfigure} \hspace*{0.1cm}
\begin{subfigure}{.23\textwidth}
\centering
\includegraphics[scale=1]{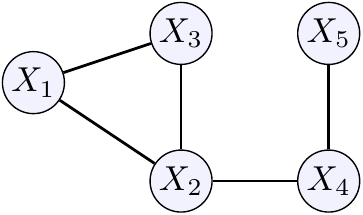}
\caption{Undirected (Moralized)}
\label{fig:moralization_b}
\end{subfigure}
\caption{Moralization of a PGM.}
\label{fig:moralization}
\end{figure}

A fundamental problem in PGMs is to estimate the marginals of each variable from the given joint distribution $p(\bx)$; this is known as the Bayesian inference problem \cite{WaiJor08}. Belief propagation (BP)~\cite{Pea88} is one of the most popular algorithms for accomplishing this task.

{\bf {Belief Propagation}:}
BP is an effective message-passing type algorithm for Bayesian inference in PGMs. It updates the marginal distribution of each node through communications of beliefs/messages between them. These messages are updated through
\begin{equation}\label{eq:BP}
	m_{i \rightarrow j} (x_j) \propto \sum_{x_i} \psi_{ij}(x_i,x_j) 
	\prod_{k\in N(i)\backslash j} m_{k \rightarrow i}(x_i),
\end{equation}
where $m_{i\rightarrow j} (x_j)$ denotes the message from variable node $i$ to variable node $j$, encapsulating the belief of node $i$ on node $j$. Here, $N(i)$ is the set of neighboring nodes of $i$, and thus $N(i)\backslash j$ denotes the set of neighbors of $i$ except for $j$. The messages in \eqref{eq:BP} are updated iteratively over the graph. When the algorithm converges, the node and edge marginals are given by
 	\begin{subequations}\label{eq:belief}
	\begin{eqnarray}
	b_i(x_i) &\propto&  \prod_{k \in N(i)} m_{k \rightarrow i} (x_i)
	\\
	\!\! \!\! \!\! \!\! \!\! \!\! \!\! b_{ij} (x_i,x_j)\!\!\!\!\! &\propto&\!\!\!\!\! \psi_{ij}(x_i,x_j)\!\!\!\!\!\! \prod_{k\in N(i) \backslash j}\!\!\!\!\!\! m_{k\rightarrow i} (x_i) \!\!\!\!\!\!\! \prod_{\ell\in N(j) \backslash i} \!\!\!\!\!\! m_{\ell \rightarrow j} (x_j\!).
	\end{eqnarray}
 	\end{subequations}
When the graph has no cycles (i.e., tree) it is well-known that the belief propagation algorithm converges globally \cite{YedFreWei01} and the estimated marginal distributions in \eqref{eq:belief} recover the true marginals exactly. For general graphs with cycles, convergence is not guaranteed but it works well in practice \cite{WaiJor08}.

Note that many time-sequence based filtering algorithms used in control and estimation, including Kalman filter \cite{Kal60}, are instances of the BP algorithm operating on Gaussian HMM, a special kind of graphical model (see Section~\ref{sec:Filtering_hmm} for more details).

\subsection{Collective Graphical Models}
\label{subsec:cgm}
CGMs were first introduced in \cite{SheDie11} as a framework for inference and learning in graphical models with aggregate data. CGMs describe the distribution of the aggregate counts of a population sampled independently from a discrete PGM. 
Assume all the individuals share the same PGM \eqref{eq:MRF}. 

Let $X^{(m)}_i$ be the random variable representing the state of the $m^{th}$ individual at node $i$. To generate the aggregate data, first assume that $M$ independent sample vectors $\bx^{(1)},...,\bx^{(M)}$  are drawn from the individual probability model to represent the individuals in a population. Here, each entry of the vector $\bx^{(m)}$ corresponding to node $i$ is $x^{(m)}_i$ that takes one of the $d$ possible states.
Let $\bn_i \in \bN^d$ be the aggregate \textit{node} distribution with entries $n_i(x_i)= \sum_{m=1}^M \mathbb{I}[X^{(m)}_i = x_i]$ that count the number of individuals in each state. Here, $\mathbb{I}[\cdot]$ denotes indicator function. Moreover, let $\bn_{ij}\in \bN^{d\times d}$ be the aggregate edge distributions with entries $n_{ij}(x_i,x_j) =  \sum_{m=1}^M \mathbb{I}[X_i^{(m)}= x_i, X_j^{(m)}= x_j]$. The vectors $\bn_1,\ldots,\bn_J$ constitute the aggregate data and the aggregate edge distributions $\bn_{ij}$ represent sufficient statistics of the individual model~\cite{SheDie11}. The collection of all the aggregate node distributions $\bn_i$ together with the aggregate edge distributions $\bn_{ij}$ is denoted as $\bn$, i.e., $\bn=\{\bn_i,\bn_{ij}\}$. 

In CGMs \cite{SheDie11}, the observation noise is modeled explicitly as a conditional distribution $p(\by|\bn)$ with $\by$ being the aggregate noisy observations that probabilistically depend on the aggregate data $\bn$. For instance, for a count $n$, the associated observation may follow a Poisson distribution ${\rm Poisson}(\beta n)$ for some coefficient $\beta>0$. The goal of inference in CGMs is to estimate $\bn$ from the aggregate noisy observations through the conditional distribution $p(\bn | \by) \propto p(\bn)p(\by|\bn)$, where $p(\bn)$ is known as the CGM distribution~\cite{SheDie11} which is derived from the individual model~\eqref{eq:MRF}. When the underlying graph is a tree, the CGM distribution $p(\bn)$ equals
\begin{equation}\label{eq:cgm_distribution}
p(\bn) = M! \frac{\prod_{i\in V} \prod_{x_i}((n_i(x_i)!)^{(d_i - 1)} }{\prod_{(i,j) \in E} \prod_{x_i,x_j} n_{ij}(x_i,x_j)!} p(\bx^{(1)},\ldots, \bx^{(M)}),
\end{equation}
where $d_i$ is the number of neighbors of node $i$ in $G$ and 
\begin{equation*}
    p(\bx^{(1)},\ldots, \bx^{(M)})\! =\! \frac{1}{Z^M}\!\! \prod_{(i,j) \in E} \prod_{x_i,x_j}\! \psi_{ij}(x_i,x_j)^{n_{ij}(x_i,x_j)}
\end{equation*}
is the joint probability of the entire population. The integer coefficient $$M! \frac{\prod_{i\in V} \prod_{x_i}((n_i(x_i)!)^{(d_i - 1)} }{\prod_{(i,j) \in E} \prod_{x_i,x_j} n_{ij}(x_i,x_j)!}$$ accounts for the fact that the individuals in aggregate observation are indistinguishable. 

The support of the CGM distribution $p(\bn)$ is such that each entry of $\bn$ is an integer and $\bn$ satisfies the following constraints
\begin{equation}\label{eq:nconstraints}
\begin{aligned}
    &\sum_{x_i}n_i(x_i) = M,\qquad\qquad ~\forall i \in V \\
    &n_i(x_i) = \sum_{x_j} n_{ij} (x_i,x_j), \quad \forall i \in V,  ~j\in N(i).
    \end{aligned}
\end{equation}

Exact inference of $\bn$, either maximum a posterior probability estimate or Bayesian inference, based on $p(\bn|\by)$ is unrealistic for large populations, since the computational complexity increases rapidly
as the population size $M$ grows \cite{SheSunKumDie13}. It was first discovered in \cite{SheSunKumDie13} that $-\ln p(\bn|\by)$ can be approximated by (up to a constant addition and multiplication) the CGM free energy 
\begin{equation}\label{eq:cgm_free_energy}
    F_{\rm CGM}(\bn) = U_{\rm CGM}(\bn) - H_{\rm CGM}(\bn),
\end{equation}
where $U_{\rm CGM}(\bn)$ equals 
\begin{equation*}
- \sum_{(i,j) \in E} \sum_{x_i,x_j} n_{ij}(x_i,x_j)  \ln \psi_{ij}(x_i,x_j)
    - \ln~p(\by|\bn),
\end{equation*}
and
\begin{equation*}
\begin{aligned}
    H_{\rm CGM}(\bn) = - \sum_{(i,j) \in E} \sum_{x_i,x_j} n_{ij}(x_i,x_j) \ln n_{ij}(x_i,x_j)  \\ + \sum_{i \in V} (d_i - 1) \sum_{x_i} n_i(x_i) ~\ln~n_{i}(x_i).
    \end{aligned}
\end{equation*}
After relaxing the constraints that $n_i(x_i), n_{ij}(x_i,x_j)$ are integers and under the assumption that the observation model $p(\by\mid\bn)$ is log-concave, the resulting problem of minimizing $F_{\rm CGM}$ is a convex optimization problem. This is the approximate MAP \cite{SheSunKumDie13} framework for CGMs. Note that the problem size of minimizing $F_{\rm CGM}$ is independent of the population size $M$. Even though the approximate MAP framework is insensitive to population size, its complexity grows rapidly as the number of variables $J$ increases. Two algorithms designed to solve the approximate MAP problems more efficiently are non-linear belief propagation (NLBP)~\cite{SunSheKum15} and Bethe regularized dual averaging (Bethe-RDA)~\cite{VilBelSheMcc15}.

\textbf{Non-linear Belief Propagation:} The NLBP~\cite{SunSheKum15} algorithm addresses the aggregate inference problem by establishing a connection between the Bethe free energy~\cite{YedFreWei05} and the objective function \eqref{eq:cgm_free_energy} for approximate MAP inference in CGMs. 
It is a message passing type algorithm for aggregate MAP inference. Roughly, it is equivalent to running standard BP on a PGM with edge potentials $\hat \psi_{ij}$ which are being updated at each iteration. The steps of the NLBP algorithm are listed in Algorithm~\ref{alg:nlbp}. 
\begin{algorithm}[h]
   \caption{Non-Linear Belief Propagation (NLBP)}
   \label{alg:nlbp}
\begin{algorithmic}
   \STATE Initialize $n_{ij}(x_i,x_j) \propto \psi_{ij}(x_i,x_j),~\forall (i,j) \in E, \forall x_i,x_j $
   \REPEAT
   \STATE Update the following in any order for all $(i,j)\in E$
           \STATE 
           \begin{eqnarray*}
           \hat{\psi}_{ij}(x_i,x_j) \!\!\!&=&\!\!\! \text{exp} \left( -  \frac{\partial U_{\rm CGM}(\bn)}{\partial n_{ij}(x_i,x_j)} \right)
            \\
            m_{i\rightarrow j}(x_j) \!\!\!&\!\!\!\propto& \sum_{x_i} \hat{\psi}_{ij}(x_i,x_j) \prod_{k \in N(i) \backslash j}m_{k \rightarrow i}(x_i)
            \\
            n_{ij}(x_i,x_j) \!\!\!&\propto&\!\!\! \hat{\psi}_{ij}(x_i,x_j) \!\!\!\!\!\! \prod_{k \in N(i) \backslash j} \!\!\!\! \!\!\! m_{k \rightarrow i}(x_i) \!\!\!\!\! \prod_{\ell \in N(j) \backslash i} \!\!\!\! \!\!\! m_{\ell \rightarrow j}(x_j)
            \end{eqnarray*}
    \UNTIL{convergence}
\end{algorithmic}
\end{algorithm}
Similar to BP, after convergence of the messages in Algorithm~\ref{alg:nlbp}, the aggregate marginals can be estimated as \begin{equation}\label{eq:nlbp_marginal}
n_i(x_i) \propto \prod_{k \in N(i)}m_{k \rightarrow i}(x_i).
\end{equation}

One of the major drawbacks of the NLBP algorithm is that it does not exhibit any convergence guarantee \cite{SunSheKum15}. The major factor affecting the convergence of NLBP is the update of the potentials $\hat{\psi}_{ij}$ which requires gradient computations of $p(\by|\bn)$. These gradients depend on the observation model at hand and might cause the explosion or saturation of potential updates based on the smoothness of the observation model $p(\by|\bn)$. To stabilize the NLBP to some degree, it was proposed \cite{SunSheKum15} to dampen the estimates $\bn$ as $\bn = (1-\alpha)\bn + \alpha \bn^{new}$ in each iteration, where $0< \alpha \leq 1$ and $\bn^{new}$ is the estimate in the current iteration. However, the selection of the parameter $\alpha$ has to be done carefully to ensure appropriate potential updates \cite{SunSheKum15}. 


{\bf Bethe regularized dual averaging:} 
Another algorithm for solving the aggregate inference problem is Bethe-RDA~\cite{VilBelSheMcc15} which is inspired by one type of proximal algorithms called regularized dual averaging (RDA)~\cite{Xia09}. The Bethe-RDA algorithm has been proven to be faster than NLBP by an order of magnitude in various experiments~\cite{VilBelSheMcc15}. Similar to NLBP, in Bethe-RDA, standard BP is used to compute the marginals at each iteration based on modified edge potentials. However, the way to update these potentials in Bethe-RDA is very different. At iteration $t$, the edge potentials $\Psi_t = \{ \psi_{ij}^t(x_i,x_j) \}$ are updated according to 
\begin{equation}
    \ln~ \Psi_t = \ln~ \Psi - \frac{t}{\beta_t + t} ~ \overline{g}_t,
\end{equation}
where $\Psi = \{ \psi_{ij}(x_i,x_j)  \}$ is the original potential, $\beta_t$ is the learning rate, and 
\begin{equation}
    \overline{g}_t = \frac{t-1}{t} ~ \overline{g}_{t-1} - \frac{1}{t}~ \frac{\partial \ln p(\by|\bn_{t-1})}{\partial \bn}.
\end{equation}
After updating the potentials, the required marginals $\bn_{t}$ are computed according to the standard BP algorithm. The steps of the Bethe-RDA algorithm are listed in Algorithm~\ref{alg:bethe_rda}.
\begin{algorithm}[h]
   \caption{Bethe-RDA}
   \label{alg:bethe_rda}
\begin{algorithmic}
   \STATE Initialize all the marginals $\bn_0 = \text{BP}[\Psi]$, and $\bar g_0 = 0$
   \REPEAT
   \vspace{0.1cm}
   \STATE $\overline{g}_t = \frac{t-1}{t} ~ \overline{g}_{t-1} - \frac{1}{t}~ \frac{\partial \ln p(\by|\bn_{t-1})}{\partial \bn}$
   \vspace{0.2cm}
   \STATE $\ln~ \Psi_t = \ln~ \Psi - \frac{t}{\beta_t + t} ~ \overline{g}_t $
   \vspace{0.2cm}
   \STATE $\bn_{t} = \text{BP}[\Psi_t]$
   \vspace{0.1cm}
    \UNTIL{convergence}
\end{algorithmic}
\end{algorithm}

\subsection{Multimarginal Optimal Transport}
\label{subsec:MOT}
In a MOT \cite{Nen16,Pas12} problem one aims to find an optimal transport plan among a set of marginal distributions, minimizing an underlying given cost function. We consider the discrete settings where the marginal distributions are described by probability vectors $\mu_j \in \mathbb{R}^d,~j \in \Gamma \subset \{1,2,\ldots,J\}$ and denote the cost function and the transport plan by the $J$-mode tensors $\bC,\bB\in\mR^{d\times d \dots \times d}$. The Kantorovich formulation of MOT with constraints on a subset of marginals $\Gamma \subset \{1,2,\dots,J\}$ reads \cite{ElvHaaJakKar20}

\begin{equation} \label{eq:omt_multi_discrete}
\begin{aligned}
\min_{\bB \in \mR_+^{d\times \dots \times d}} ~&~ \langle \bC, \bB \rangle  \\
\text{subject to}~~ & P_j (\bB) = \mu_j,  \text { for } j \in \Gamma,
\end{aligned}
\end{equation}
where $\langle \bC, \bB \rangle = \sum_{i_1,\dots,i_J} \bC_{i_1,\dots,i_J} \bB_{i_1,\dots,i_J}$, and the projection on the $j$-th marginal of $\bB$ is defined by
\begin{equation} \label{eq:proj_discrete}
P_j(\bB) = \sum_{i_1,\dots,i_{j-1},i_{j+1},i_J} \bB_{i_1,\dots,i_{j-1},i_j,i_{j+1},\dots,i_J}.
\end{equation}

Note that the standard OT problem with two marginals is a special case of \eqref{eq:omt_multi_discrete} with $J=2$ and $\Gamma=\{1,2\}$. 

Though \eqref{eq:omt_multi_discrete} is a standard linear programming, it can be computational expensive for large $d$. For faster computations, it was proposed by \cite{Cut13,BenCarCut15,CheGeoPav16} to add a regularizing entropy term
\begin{equation}
H(\bB) = - \sum_{i_1,\dots,i_J} \bB_{i_1,\dots,i_J}  \ln~(\bB_{i_1,\dots,i_J})
\end{equation}
to the objective.
This results in the strongly convex problem
\begin{equation} \label{eq:omt_multi_regularized}
\begin{aligned}
\min_{\bB \in \mR_+^{d\times \dots \times d} }~ &~ \langle \bC, \bB \rangle - \epsilon H(\bB) \\
\text{subject to}~~ & P_j (\bB) = \mu_j,  \text { for } j \in \Gamma,
\end{aligned}
\end{equation}
where $\epsilon>0$ is a regularization parameter. 

It can be shown that the unique optimal solution to \eqref{eq:omt_variational} is of the form
\begin{equation}
\bB = \bK \odot \bU,
\end{equation}
where $\bK = \exp(- \bC/\epsilon)$ and $\bU= u_1 \otimes u_2 \otimes \dots \otimes u_J$ with $u_j = \mathbf{1}$ ($\mathbf{1}$ denotes the vector of all entries being $1$) for $j\notin \Gamma$. Here $\otimes$ denotes tensor product and $\exp$ is entry-wise exponential map. The Sinkhorn scheme \cite{DemSte40,Sin64,Cut13} for finding the vectors $u_j$, given that they are initialized to be $\mathbf{1}$,
is to iteratively update them according to
\begin{equation} \label{eq:sinkhorn_multi}
u_j \leftarrow u_j \odot \mu_j ./ P_j(\bK \odot \bU),
\end{equation}
for all $j\in\Gamma$. Here $\odot$ and $./$ denote entry-wise multiplication and division, respectively. It is worth noting that the update step \eqref{eq:sinkhorn_multi} in Algorithm \ref{alg:sinkhorn} is a \textit{scaling} step that ensures that the $j$-th marginal of the updated tensor $ \bK \odot \bU$ satisfies the constraint in \eqref{eq:omt_multi_regularized}, i.e., $ P_j (\bK \odot \bU) = \mu_j$. The steps of Sinkhorn scheme are summarized in Algorithm~\ref{alg:sinkhorn} for future reference. 
\begin{algorithm}[h]
   \caption{Sinkhorn Algorithm for MOT}
   \label{alg:sinkhorn}
\begin{algorithmic}
    \STATE Compute $\bK=\exp(- \bC/\epsilon)$
   \STATE Initialize $u_1, u_2, \ldots, u_J$ to $\mathbf{1}$
   \REPEAT
   \FOR{$j \in \Gamma$}
        \STATE $\bU \leftarrow u_1 \otimes u_2 \otimes \dots \otimes u_J$
        \STATE
        $u_j \leftarrow u_j \odot \mu_j ./ P_j(\bK \odot \bU)$
    \ENDFOR
    \UNTIL{convergence}
\end{algorithmic}
\end{algorithm}
The Sinkhorn algorithm may for instance be derived as Bregman iterations \cite{BenCarCut15} or dual block coordinate ascend \cite{ElvHaaJakKar20}. It has global convergent guarantee with linear rate \cite{LuoTse92,HaaRinCheKar20}. We remark that even though Sinkhorn algorithm has linear convergent rate, the cost for each iteration could be high due to the projection $P_j$, whose complexity scales exponentially with $J$.


\section{Inference with aggregate data}
\label{sec:main_result}
We consider Bayesian marginal inference problems with aggregate data as in CGMs with a different observation model. We reformulate them into MOT problems and then leverage Sinkhorn algorithm (Algorithm \ref{alg:sinkhorn}) to develop an efficient algorithm for our problems.

\subsection{Problem formulation}
Assume that the graph $G = (V,E)$ encodes the relationships among the node variables  $X_1,X_2,\ldots,X_J$ of each individual, which consists of unobserved as well as observed variables, and let $\Gamma\subset V$ be the set of observation nodes. Let the unobserved individual variables take values in a finite set $\mathcal{X}_u$ and the observations come from another finite set $\mathcal{X}_o$, where in general, $\mathcal{X}_u \neq \mathcal{X}_o$. The joint distribution of individual variables is assumed to be factored as \eqref{eq:MRF}.

Then similar to the generative model of aggregate data in CGMs, by drawing $M$ independent samples from the individual model, the aggregate counts corresponding to each node is generated. Therefore, the aggregate data constitute $\bn_1, \bn_2, \ldots, \bn_J$ (here $J=|V|$). We assume that the system is closed~\cite{KalLawVol83}, i.e., the population size $M$ remains fixed. In such closed settings, the aggregate data can be thought of as probability distributions when \textbf{normalized} with the population size. 
With this setup, when the underlying graph structure is a {\em tree}, the aggregate variables have the same graph structure as of the individual probability model; this is due to the hyper-Markov property~\cite{SheSunKumDie13}. Now, we have the aggregate distributions constituting $\bn_1,\bn_2,\ldots, \bn_J$ with the underlying structure $G$. Suppose aggregate observations are made from a subset of nodes $\Gamma \subset V$ and denote these aggregate observation by $\by_i, \forall i \in \Gamma$, then our goal is to infer the aggregate marginals $\bn_i, ~\forall i \notin \Gamma$.
\begin{rem}
Without loss of generality, we assume that the observation nodes can only be leaves, otherwise, we can always split the underlying tree-structured graph over the observation node and then each subgraph will have this observation node as a leaf.
\end{rem}
\begin{figure}[tb]
\centering
\begin{subfigure}{.23\textwidth}
\centering
\includegraphics[scale=0.63]{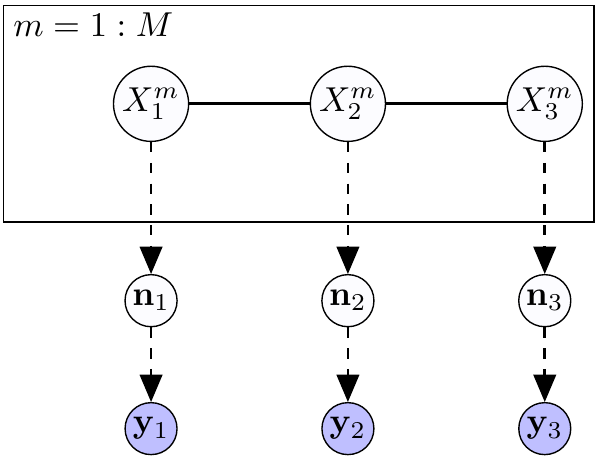}
\caption{} \label{fig:noise_model_a}
\end{subfigure} \hspace*{0.1cm}
\begin{subfigure}{.23\textwidth}
\centering
\includegraphics[scale=0.63]{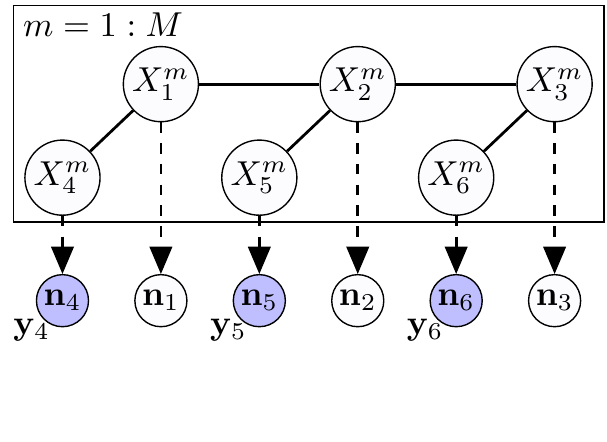}
\caption{} \label{fig:noise_model_b}
\end{subfigure}
\caption{Different aggregate observation models (shaded nodes represent aggregate observations): (a) CGMs model the aggregate noisy observations explicitly and (b) in our model, the observation noise is encoded in the underlying graphical model.}
\label{fig:noise_model}
\vspace{-11pt}
\end{figure}
Note that in our setting, the observation model is a subset of the underlying graph as opposed to the original CGM setting, where the observation model is treated separately. Figure~\ref{fig:noise_model} depicts this difference between the noise models. In \cite{SheSunKumDie13}, the observation model we use here was regarded as exact observation since the aggregate measurements are exact marginal distributions of the associated node variable. We argue that this type of observation can also handle measurement noise. Taking Figure~\ref{fig:noise_model} as an example, $X_4$ is treated as a measurement node of $X_1$, therefore, the measurement noise is already encoded in the edge potential between them. Indeed, this is the measurement noise model used in standard HMMs \cite{FinSinTis98}; the measurement noise is captured by the emission probability, which is the edge potential between a hidden state node and an observation node (see Section \ref{sec:Filtering_hmm} for more details).
As a side note, the algorithms developed in \cite{SheSunKumDie13,SunSheKum15,VilBelSheMcc15} do not apply to the cases with ``exact observation''. Taking the limit of those algorithms designed for ``noisy observation'' with vanishing noise will cause ill-conditioning issues in the updates of the algorithms. 

This can be found by maximizing the posterior distribution $p(\by|\bn) \propto p(\bn,\by)$, which is equivalent to minimizing the negative logarithm $-\ln p(\bn, \by)$. Since we model the observation variables within the graphical model as in Figure~\ref{fig:noise_model_b}, the observation noise is implicitly incorporated in the CGM distribution given by Equation~\eqref{eq:cgm_distribution}. 




Using Equation~\eqref{eq:cgm_distribution}, we arrive at the following
\begin{align*}
   - \ln~p(\bn,\by) =& - \ln~M! - \sum_{i \in V} (d_i - 1) \sum_{x_i} \ln ( n_i(x_i)!) \\
    & + \sum_{(i,j)\in E} \sum_{x_i,x_j} \ln ( n_{ij}(x_i,x_j) !)  +  M \ln~Z \\
    & - \sum_{(i,j)\in E} \sum_{x_i,x_j} n_{ij}(x_i,x_j) \ln~\psi_{ij}(x_i,x_j).
\end{align*}
Again, as in CGM, it is computationally intractable to directly minimize the negative logarithm $- \ln~p(\bn,\by)$ due to the presence of factorial terms and $\bn$ being integers. By invoking the Stirling approximation $\ln( a!)=  a~\ln~a - a + \mathcal{O}(\ln~a)$ as in \cite{SheSunKumDie13} (see also \cite[Prop.~1]{HasRinChe19}), we obtain
\begin{align*}
   - \ln~p(\bn,\by) =&  - \sum_{i \in V} (d_i - 1) \sum_{x_i} n_i(x_i)  ~\ln ~n_i(x_i) \\
    & + \sum_{(i,j)\in E} \sum_{x_i,x_j} n_{ij}(x_i,x_j) ~\ln~n_{ij}(x_i,x_j) \\
    & - \sum_{(i,j)\in E} \sum_{x_i,x_j} n_{ij}(x_i,x_j) \ln( \psi_{ij}(x_i,x_j)) \\
    & - M \ln(M/Z)+ \mathcal{O}( \ln M).
\end{align*}

Denote $\hat \bn_i = \bn_i/M$ and $\hat \bn_{ij} = \bn_{ij}/M$ the normalizations of the aggregate distributions, then it follows
\begin{align*}
   - \frac{1}{M} \ln~p(\bn,\by)  = &  - \sum_{i \in V} (d_i - 1) \sum_{x_i} \hat n_i(x_i)  ~\ln ~\hat n_i(x_i) \\
    & \hspace{-1cm}+ \sum_{(i,j)\in E} \sum_{x_i,x_j} \hat n_{ij}(x_i,x_j) ~\ln~\hat n_{ij}(x_i,x_j)+\ln Z \\
    & \hspace{-2cm} - \sum_{(i,j)\in E} \sum_{x_i,x_j} \hat n_{ij}(x_i,x_j) \ln( \psi_{ij}(x_i,x_j)) + \mathcal{O}(\frac1M\ln M).
\end{align*}
Thus, when $M$ is large, the (scaled) log-likelihood converges to a quantity depending only on the normalized marginals $\hat\bn_i, \hat \bn_{ij}$. For the sake of simplicity, from now on, we use the $\bn_i, \bn_{ij}$ instead of $\hat\bn_i, \hat \bn_{ij}$ to denote the normalized marginal distributions. The aggregate observations $\by_i, ~\forall i \in \Gamma$ are also normalized so that $\sum_{x_i} y_i(x_i) = 1$. We further denote their joint distributions by $\mN$. Thanks to the tree structure of the underlying graph, minimizing the above is equivalent to the following problem.
\begin{prob}\label{prob:variational}
	\begin{subequations}\label{eq:variational}
	\begin{eqnarray}\label{eq:variational1}
    		\min_{\mN} && {\rm KL}(\mN~ ||~ \prod_{(i,j)\in E} \psi_{ij}(x_i,x_j))
		\\ \mbox{ subject to} &&P_j(\mN) = \by_j, \quad \forall j \in \Gamma,  \label{eq:variational2}
	\end{eqnarray}
	\end{subequations}
where $P_j$ denotes projection operation, i.e., $\bn_j = P_j(\mN)$.
\end{prob}
Problem \ref{prob:variational} is similar to the variational inference formulation in standard PGM~\cite{YedFreWei05} except for the existence of the extra constraints \eqref{eq:variational2}. Such a constrained version of variational inference has indeed been studied in \cite{TehWel02}, however, from a very different perspective. Problem \ref{prob:variational} provides a new viewpoint for our CGM inference problem: finding a joint aggregate distribution $\mN$ that is closest to the prior model of the PGM while satisfying the marginal constraints \eqref{eq:variational2}. More rigorously, the equivalence between Problem \ref{prob:variational} and the problem of finding the most likely $\mN$ is a consequence of the theory of large deviation \cite{Var84}. A special case of Problem \ref{prob:variational} that has been widely studied is the Schr\"odinger bridge problem \cite{Leo14,CheGeoPav14e}, which corresponds to the choice $J=2$ and $\Gamma = \{1,2\}$. Thus, Problem \ref{prob:variational} can also be viewed as a multi-marginal generalization of the Schr\"odinger bridge problems. 



\subsection{Multimarginal Optimal Transport Approach}
\label{subsec:cgm_MOT}
When treating the pairwise potentials of the graphical model as the local components of the cost function in MOT, i.e.,
\begin{equation}\label{eq:cost_structure}
    \bC(\bx) = -\sum_{i,j}\ln \psi_{ij} (x_i,x_j), 
\end{equation}
Problem \eqref{eq:variational} can be viewed as a regularized MOT problem. Indeed, the regularized MOT problem \eqref{eq:omt_multi_regularized} can be rewritten as 
\begin{equation}
\label{eq:omt_variational}
\begin{aligned}
    & \underset{\bB \in \mR_+^{d\times \dots \times d}}{\text{min}} && \text{KL} \left( \bB~ ||~ \text{exp} \left( - \frac{\mathbf{C}}{\epsilon} \right) \right) \\
    & \text{subject to} && P_j(\bB) = \mu_j,\quad \forall j \in \Gamma.
\end{aligned}
\end{equation}
Plugging \eqref{eq:cost_structure} into the above and taking $\epsilon=1$ yields exactly the same expression as in \eqref{eq:variational1}.

Consequently, we can adopt the Sinkhorn algorithm (Algorithm~\ref{alg:sinkhorn}) to solve the MOT Problem~\ref{prob:variational}, and 
this is guaranteed to converge~\cite{FraLor89}. Thanks to the graphical structure \eqref{eq:cost_structure} of the cost $\bC$, we can further accelerate the algorithm by utilizing standard BP to realize the key projection step $P_j(\bK \odot \bU)$ in the Sinkhorn algorithm; we call this combination of Sinkhorn and BP the Sinkhorn belief propagation (SBP) algorithm. 

Before presenting our SBP algorithm, we first characterize the stationary points of the optimization \eqref{eq:variational} in terms of local messages; these will be used to compute the projections in a Sinkhorn scaling step. When the underlying graph is a tree, the objective function of \eqref{eq:variational} is the same as the Bethe free energy~\cite{YedFreWei05}
\begin{equation}\label{eq:Bethe_energy_MOT}
\begin{aligned}
    F_{\rm Bethe}(\bn) =  \sum_{(i,j) \in E} \sum_{x_i,x_j} n_{ij}(x_i,x_j) \ln \frac{n_{ij}(x_i,x_j)}{\psi_{ij}(x_i,x_j)} \\
    - \sum_{i} (d_i - 1) \sum_{x_i} n_i(x_i) \ln n_i(x_i).
    \end{aligned}
    \end{equation}
Thus, Problem \ref{prob:variational} 
takes the form
\begin{subequations}\label{eq:MOT_bethe}
\begin{eqnarray}
    \min_{\bn_{ij},\bn_i}\quad  &&F_{\rm Bethe}(\bn) \label{eq:MOT_bethe_a}
    \\
   \text{subject to} && n_i(x_i) = y_i(x_i),~ \forall i \in \Gamma \label{eq:MOT_bethe_b} \\
   && \sum_{x_j} n_{ij}(x_i,x_j) = n_i(x_i), \forall (i,j) \in E \label{eq:MOT_bethe_c}  \\
   && \sum_{x_i} n_i(x_i)= 1, ~ \forall i \in V.\label{eq:MOT_bethe_d}
\end{eqnarray}
\end{subequations}
Here \eqref{eq:MOT_bethe_b} corresponds to aggregate observation constraints and \eqref{eq:MOT_bethe_c}-\eqref{eq:MOT_bethe_d} represent consistency constraints.
One can apply Lagrangian duality theory \cite{BoyVan04} to the constrained convex optimization \eqref{eq:MOT_bethe} to obtain the following. The proof is deferred to the Appendix. 
\begin{thm} \label{thm:is_bp}
The solution to the aggregate inference problem~\eqref{eq:MOT_bethe} is characterized by 
\begin{equation}\label{eq:nmarginal}
    n_i (x_i)  \propto  \prod_{k \in N(i)}  m_{k \rightarrow i}(x_i), ~\forall i\notin \Gamma
\end{equation}
where $m_{i \rightarrow j}(x_j)$ are fixed points of
\begin{subequations}\label{eq:is_bp_MOT}
\begin{eqnarray}
    m_{i\rightarrow j} (x_j)  &\propto&  \sum_{x_i} \psi_{ij}(x_i,x_j) \prod_{k\in N(i)\backslash j}  m_{k\rightarrow i}(x_i); \nonumber \\
    &&\forall i \notin \Gamma,~  \forall j \in N(i), \label{eq:is_bp_MOT1}  \\
    m_{i\rightarrow j} (x_j) &\propto&  \sum_{x_i} \psi_{ij}(x_i,x_j) \frac{y_i(x_i)}{m_{j \rightarrow i} (x_i)}; \nonumber \\
    && \forall i \in \Gamma,~  \forall j \in N(i).  \label{eq:is_bp_MOT2}
\end{eqnarray}
\end{subequations}
\end{thm}

The expression $m_{i\rightarrow j}$ in~\eqref{eq:is_bp_MOT} can be viewed as messages between nodes, as in BP. Among the two classes of messages in~\eqref{eq:is_bp_MOT}, \eqref{eq:is_bp_MOT1} resembles that in standard BP while \eqref{eq:is_bp_MOT2} corresponds to the scaling step \eqref{eq:sinkhorn_multi}.

Taking all these components into account, we arrive at the Sinkhorn belief propagation algorithm (Algorithm~\ref{alg:is_bp}). Once converged, the marginals can be recovered by \eqref{eq:nmarginal}. The initialization runs a standard BP algorithm over the same PGM without accounting for the constraints \eqref{eq:variational2}. Let $\bar \Gamma$ be a sequence containing elements in $\Gamma$ such that all neighboring elements are different and each element in $\Gamma$ appear in $\bar \Gamma$ infinitely often. For an element $i\in \bar \Gamma$, denote the element following $i$ by $i_{\rm next}$.  
\begin{algorithm}[h]
   \caption{Sinkhorn Belief Propagation (SBP)}
   \label{alg:is_bp}
\begin{algorithmic}
   \STATE Initialize all the messages $m_{i\rightarrow j} (x_j)$ 
   \WHILE{not converged}
   \FOR{$i \in \bar \Gamma$}
        \STATE i) Update  $m_{i\rightarrow j} (x_j)$ using \eqref{eq:is_bp_MOT2}
        \STATE ii) Update all the messages on the path from $i$ to $i_{\rm next}$ 
        according to \eqref{eq:is_bp_MOT1}
    \ENDFOR
    \ENDWHILE
\end{algorithmic}
\end{algorithm}

In the language of the Sinkhorn algorithm (Algorithm \ref{alg:sinkhorn}), $m_{i\rightarrow j}$ is associated with $u_i$, thus step i) corresponds to modifying the PGM potential from $\bK$ to $\bK\odot\bU$ with the most recent $\bU$. The update ii) then calculates the marginal distribution at node $i_{\rm next}$ of the PGM with this modified potential. Due to the tree structure of the PGM, it suffices to update only messages from node $i$ to $i_{\rm next}$ as in step ii) of Algorithm~\ref{alg:is_bp}. 
This can be easily explained via an example as depicted in Figure~\ref{fig:sbp_explain} with $i=1, i_{\rm next} = 2$. After updating the message $m_{1\rightarrow 4} (x_4)$, we only need to update $m_{4\rightarrow 6} (x_6)$ and $m_{6\rightarrow 2} (x_2)$ since these are the only two messages that contribute to the next scaling update of $m_{2\rightarrow 6} (x_6)$ as explained by \eqref{eq:is_bp_MOT2}.

\begin{figure}[t]
\centering
\includegraphics[scale=1]{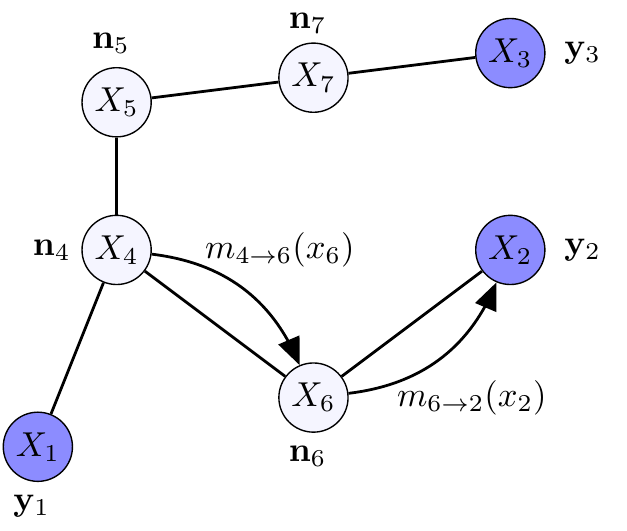}
\caption{Example graph.}
\label{fig:sbp_explain}
\end{figure}

The Sinkhorn algorithm (Algorithm \ref{alg:sinkhorn}) has a linear convergence rate \cite{LuoTse93}, so does SBP. Each iteration in Algorithm \ref{alg:sinkhorn} requires a projection step $P_j$, which is realized by belief propagation between neighboring observation nodes $i$ and $i_{\rm next}$. For a graph with $J$ nodes, where each node takes $d$ possible values, the complexity of operation \eqref{eq:is_bp_MOT1} is $O(d^2)$. The update ii) in SBP takes at most $J$ numbers of operation \eqref{eq:is_bp_MOT1}, thus, the worst case iteration complexity of SBP is $O(Jd^2)$.

\section{Filtering over collective hidden Markov models}
\label{sec:Filtering_hmm}

In this section, we study filtering problems in a hidden Markov model with aggregate observations. Towards this, we present the collective forward-backward algorithm (Algorithm \ref{alg:forward_backward}) which is a special case of SBP in HMMs. We also discuss the connections between collective filtering and filtering in standard (individual) HMMs.

An HMM consists of a hidden (unobservable) Markov chain that evolves over time and corresponding noisy variables that are observed. For the sake of simplicity in notation, denote the unobserved variables as $X_1,X_2,\ldots$ and observed variables as $O_1,O_2,\ldots$. Therefore, the underlying graph consists of the variables $V = \{X_1,X_2,\ldots, O_1,O_2,\ldots\}$ with $\Gamma = \{O_1,O_2,\ldots\}$. An HMM is parameterized by the initial distribution $\pi(X_1)$, the state transition probabilities $p(X_{t+1} \mid X_t)$, and the observation probabilities $p(O_{t} \mid X_{t})$ for each time step $t = 1,2,\ldots$. An HMM of length $T$ is represented graphically as shown in Figure~\ref{fig:hmm_model}.

\begin{figure}[h]
\centering
\includegraphics[scale=1]{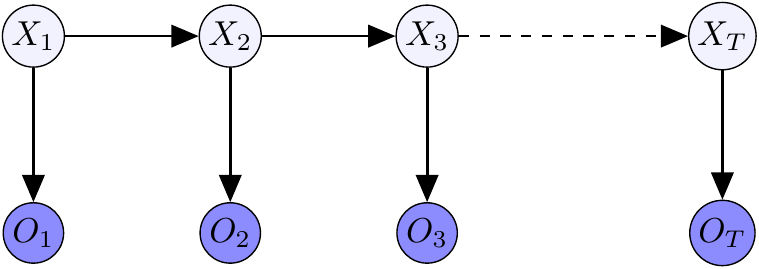}
\caption{Graphical representation of a length $T$ HMM.}
\label{fig:hmm_model}
\end{figure}

The joint distribution of an HMM with length $T$ can be factorized as
\begin{equation} \label{eq:HMM_distribution}
    p(\bx, \bo) = \pi(x_1)~ \prod_{t=1}^{T-1} ~p(x_{t+1} \mid x_{t}) ~ \prod_{t=1}^{T}~p(o_t \mid x_t),
\end{equation}
where $\bx = \{x_1,x_2,\ldots, x_T\}$ and $\bo = \{o_1,o_2,\ldots, o_T\}$ represent a particular assignment of hidden and observation variables respectively, $\pi(x_1)$ is the initial distribution of the starting state, $p(x_{t+1} \mid x_{t})$ are the transition probabilities between hidden variables, and $p(o_{t} \mid x_{t})$ are the emission probabilities for time steps $t=1,2,\ldots, T$. Although the graphical model of HMM depicted in Figure~\ref{fig:hmm_model} is directed, it can be equivalently represented by an undirected model (Markov random field) via the moralization technique as discussed in Section~\ref{subsec:pgm}. It turns out that the corresponding undirected model for an HMM is found by just replacing all the directed edges with undirected ones as in Figure~\ref{fig:hmm_message}. More specifically, the factorization of the joint distribution described by~\eqref{eq:HMM_distribution} in terms of transition and emission probabilities is equivalent to the factorization in~\eqref{eq:MRF} with edge potentials being the corresponding conditional probabilities, e.g., $\psi_{t,t} (x_t,o_t) = p(o_{t} \mid x_{t})$ and $\psi_{t,t+1} (x_t,x_{t+1}) = p(x_{t+1} \mid x_{t})$.

\subsection{Collective forward-backward algorithm}
\label{subsec:forward_backward}

To generate aggregate data in HMM settings, we follow the generative model discussed in Section~\ref{subsec:cgm} using trajectories of a large number ($M$) of individuals. Let us denote the messages in collective HMM as shown in Figure~\ref{fig:hmm_message}, where $\alpha_t(x_t)$ are the messages in the forward direction and $\beta_t(x_t)$ are the messages in the backward direction. Moreover, $\gamma_t(x_t)$ denote the messages from observation node to hidden node and $\xi_t(o_t)$ are the messages from hidden node to observation node. Note that the forward messages are $\alpha_t(x_t) = m_{t-1 \rightarrow t} (x_t)$ and the backward messages are $\beta_t(x_t) = m_{t+1 \rightarrow t} (x_t)$. Moreover, the upward and downward messages correspond to $\gamma_t(x_t) = m_{t \rightarrow t} (x_t)$ and $\xi_t(o_t) = m_{t \rightarrow t} (o_t)$, respectively. By abuse of notation, here both the upward and downward messages are denoted as $m_{t \to t}$, and distinguished only by the argument. Based on Theorem~\ref{thm:is_bp}, these messages are characterized via the following. 
\begin{cor}\label{proposition:forward_backward}
The solution to aggregate filtering problem (Problem \ref{prob:variational}) in a collective HMM is
\begin{equation}
    n_t(x_t) \propto \alpha_t(x_t) \beta_t(x_t) \gamma_t(x_t),~~ \forall t =1,2\ldots,T
\end{equation}
where $\alpha_t(x_t), \beta_t(x_t),$ and $\gamma_t(x_t)$ are the fixed points of the following updates

\begin{subequations}\label{eq:forward_backward}
\begin{eqnarray}
    \alpha_t(x_t) \propto \sum_{x_{t-1}} p(x_t|x_{t-1}) \alpha_{t-1} (x_{t-1}) \gamma_{t-1}(x_{t-1}), \label{eq:forward_backward1}  \\
    \beta_t(x_t) \propto  \sum_{x_{t+1}} p(x_{t+1}|x_{t}) \beta_{t+1} (x_{t+1}) \gamma_{t+1}(x_{t+1}), \label{eq:forward_backward2} \\
    \gamma_t(x_t) \propto \sum_{o_{t}} p(o_{t}|x_{t}) \frac{y_t(o_t)}{\xi_t(o_t)}, \label{eq:forward_backward3} \\
    \xi_t(o_t) \propto \sum_{x_{t}} p(o_{t}|x_{t}) \alpha_{t} (x_{t}) \beta_{t}(x_{t}), \label{eq:forward_backward4}
\end{eqnarray}
\end{subequations}
with boundary conditions 
\begin{equation}
    \alpha_1(x_1) = \pi(x_1) \quad \text{and} \quad \beta_T(x_T) = 1.
\end{equation}
\end{cor}
\begin{figure}[t]
\centering
\includegraphics[scale=0.8]{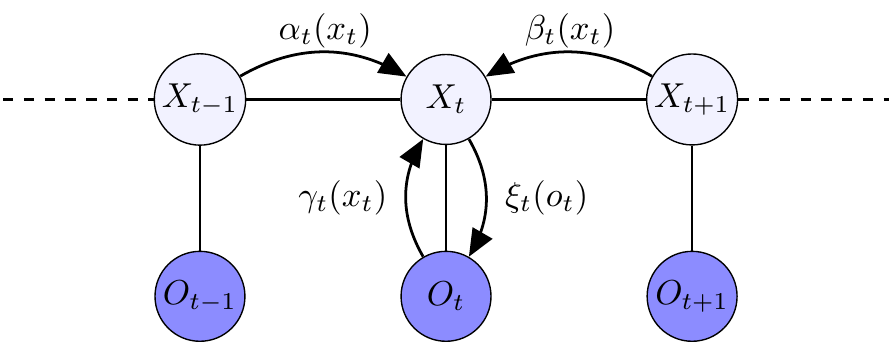}
\caption{Messages for inference in collective HMM.}
\label{fig:hmm_message}
\end{figure}

\begin{algorithm}[h]
   \caption{Collective forward-backward algorithm}
   \label{alg:forward_backward}
\begin{algorithmic}
   \STATE Initialize all the messages $\alpha_t(x_t), \beta_t(x_t), \gamma_t(x_t), \xi_t(o_t)$ 
   \WHILE{not converged}
   \STATE \textbf{Forward pass:}
   \FOR{$t = 2,3,\ldots,T$}
        \STATE i) Update  $\gamma_{t-1}(x_{t-1})$ using \eqref{eq:forward_backward3}
        \STATE ii) Update $\alpha_t(x_t), \xi_t(o_t)$
    \ENDFOR
    \STATE \textbf{Backward pass:}
    \FOR{$t = T-1,\ldots,1$}
        \STATE i) Update  $\gamma_{t+1}(x_{t+1})$ using \eqref{eq:forward_backward3}
        \STATE ii) Update $\beta_t(x_t), \xi_t(o_t)$
    \ENDFOR
    \ENDWHILE
\end{algorithmic}
\end{algorithm}
Combining Corollary \ref{proposition:forward_backward} and Algorithm \ref{alg:is_bp} we arrive at the collective forward-backward algorithm (Algorithm~\ref{alg:forward_backward}). Since the graph underlying in HMM is a tree, Algorithm~\ref{alg:forward_backward} is guaranteed to converge and the estimated marginals are exact. Upon the convergence of the algorithm, the marginals can be estimated as
\begin{equation*}
    n_t(x_t)  \propto \alpha_t(x_t) \beta_t(x_t) \gamma_t(x_t),~~ \forall t =1,2\ldots,T.
\end{equation*}
Note that Algorithm~\ref{alg:forward_backward} uses the scheduling sequence $\bar\Gamma=\{o_1,o_2,\ldots,o_T,o_{T-1},\ldots,o_1,o_2,\ldots\}$. Other choice of $\bar\Gamma$ would also work. 
Next, we discuss the scenario of a single trajectory with $M=1$ and show that, in such case, Algorithm~\ref{alg:forward_backward} reduces to the standard forward-backward algorithm \cite{Rab89} which is widely used in filtering problems for HMMs. 

\subsection{Connections to standard HMM filtering}
\label{subsec:connection_HMM}

In standard HMM, a fixed observation sample is recorded at each time step $t$ forming a $T$ length sequence constituting an individual's observations. This filtering/smoothing task in standard HMM is achieved using the standard BP algorithm discussed in Section~\ref{subsec:pgm}. In this special case of HMM graph as depicted in Figure~\ref{fig:hmm_model}, the standard BP takes a simple form, known as forward-backward algorithm \cite{Rab89}. 

\begin{algorithm}[t]
   \caption{Forward-backward algorithm}
   \label{alg:standard_forward_backward}
\begin{algorithmic}
   \STATE Initialize the messages $\alpha_t(x_t), \beta_t(x_t)$ with $\alpha_1(x_1) = \pi(x_1)$ and $\beta_T(x_T) = 1$
   \STATE \textbf{Forward pass:}
   \FOR{$t = 2,3,\ldots,T$}
        \STATE Update $\alpha_t(x_t)$ using \eqref{eq:forward_standard}
    \ENDFOR
    \STATE \textbf{Backward pass:}
    \FOR{$t = T-1,\ldots,1$}
        \STATE Update $\beta_t(x_t)$ using \eqref{eq:backward_standard}
    \ENDFOR
\end{algorithmic}
\end{algorithm}

The required marginals, given the observation sequence $\hat{o}_{1:T} = \{\hat{o}_1,\hat{o}_2,\ldots,\hat{o}_T\}$, are
\begin{align}
    p(x_t|\hat{o}_{1:T}) &\propto p(\hat{o}_{1:T} | x_t) p(x_t) \nonumber \\
    &= p(\hat{o}_{1:t-1} | x_t) p(\hat{o}_{t} | x_t) p(\hat{o}_{t+1:T} | x_t) p(x_t) \nonumber \\
    &= p(\hat{o}_{t} | x_t) p(\hat{o}_{1:t-1}, x_t)  p(\hat{o}_{t+1:T} | x_t) \nonumber \\
    &= p(\hat{o}_{t} | x_t) \alpha_t(x_t) \beta_t(x_t),
\end{align}
where $\alpha_t(x_t) = p(x_t,\hat{o}_{1:t-1})$ denote the forward messages and $\beta_t(x_t) = p(\hat{o}_{t+1:T} | x_t)$ represent the backward messages. These messages in the forward-backward algorithm take the form
\begin{align}\label{eq:forward_standard}
    \alpha_t(x_t) &= p(x_t,\hat{o}_{1:t-1}) \nonumber \\ 
    &= \sum_{x_{t-1}} p(x_t, x_{t-1},\hat{o}_{1:t-1}) \nonumber \\
    &= \sum_{x_{t-1}} p(\hat{o}_{t-1}|x_{t-1}) p(x_t|x_{t-1},\hat{o}_{1:t-2}) p (x_{t-1},\hat{o}_{1:t-2})  \nonumber \\
    &= \sum_{x_{t-1}} p(x_t|x_{t-1}) \alpha_{t-1}(x_{t-1}) p(\hat{o}_{t-1}|x_{t-1}),
\end{align}
\begin{align}\label{eq:backward_standard}
    \beta_t(x_t) &= p(\hat{o}_{t+1:T} | x_t) \nonumber \\ 
    &= \sum_{x_{t+1}} p(\hat{o}_{t+1},\hat{o}_{t+2:T},x_{t+1} | x_t)  \nonumber \\
    &= \sum_{x_{t+1}} p(x_{t+1}|x_t)  p(\hat{o}_{t+2:T} | x_{t+1}) p(\hat{o}_{t+1}|x_{t+1}) \nonumber \\
    &= \sum_{x_{t+1}} p(x_{t+1}|x_t)  \beta_{t+1}(x_{t+1}) p(\hat{o}_{t+1}|x_{t+1}).
\end{align}

All the steps of the standard forward-backward algorithm are listed in Algorithm~\ref{alg:standard_forward_backward}. Now we establish the relationship between filtering in standard and collective HMMs.

\begin{thm}
For a single individual case ($M=1$), the collective forward-backward algorithm reduces to the standard forward-backward algorithm.
\end{thm}
\begin{proof}
For a single individual case ($M=1$), a fixed sequence of observations $\hat{o}_1, \hat{o}_2,\ldots,\hat{o}_T$ is made. In this case, the (aggregate) observations take the form 
\begin{equation}
    y_{t}(o_t) = \delta(o_t - \hat{o}_t),
\end{equation}
where $\delta(\cdot)$ denotes Dirac function. With these fixed delta observations, the messages in Equation~\eqref{eq:forward_backward3} become
\begin{equation*}
    \gamma_t(x_t) \propto \sum_{o_{t}} p(o_{t}|x_{t}) \frac{\delta(o_t - \hat{o}_t)}{\xi_t(o_t)} = \frac{p(\hat{o}_t | x_{t})}{\xi_t(\hat{o}_t)}.
\end{equation*}
Note that the denominator in the last term of the above equation can be omitted since it only serves as a scaling coefficient and therefore,
\begin{equation} \label{eq:upward_delta}
     \gamma_t(x_t) = p(\hat{o}_t | x_{t}).
\end{equation}
Now the forward and backward messages take the form
\begin{subequations}\label{eq:standard_forward_backward}
\begin{eqnarray}
    \alpha_t(x_t) \propto \sum_{x_{t-1}} p(x_t|x_{t-1}) \alpha_{t-1} (x_{t-1}) p(\hat{o}_{t-1} | x_{t-1}), \label{eq:standard_forward_backward1}  \\
    \beta_t(x_t) \propto \sum_{x_{t+1}} p(x_{t+1}|x_{t}) \beta_{t+1} (x_{t+1}) p(\hat{o}_{t+1} | x_{t+1}). \label{eq:standard_forward_backwar2} 
\end{eqnarray}
\end{subequations}
Using which, the required marginals are estimated as
\begin{equation}
    n_t(x_t) \propto p(\hat{o}_t | x_{t}) \alpha_t(x_t) \beta_t(x_t).
\end{equation}
The messages given by Equation~\eqref{eq:standard_forward_backward} are nothing but the messages used in the forward-backward algorithm for standard HMMs as in Algorithm~\ref{alg:standard_forward_backward}. 
\end{proof}

The relationship between the filtering problems in standard and collective HMMs can be intuitively explained as depicted in Figure~\ref{fig:relationship}. In case of general aggregate observations, the downward messages $\xi_{t}(o_t)$ bounce back (Figure~\ref{fig:relationship_a}) and contribute to the corresponding upward messages $\gamma_t(x_t)$ as in Equation~\eqref{eq:forward_backward3}. In case of $M=1$, the aggregate observations result in Dirac distributions and as a consequence, the downward messages get absorbed (Figure~\ref{fig:relationship_b}) and do not contribute to upward messages as explained by Equation~\eqref{eq:upward_delta}.

\begin{figure}[tb]
\centering
\begin{subfigure}{.23\textwidth}
\centering
\includegraphics[scale=0.63]{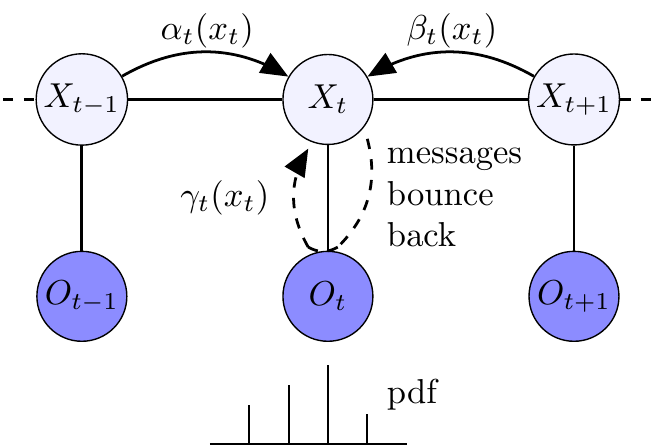}
\caption{Aggregate observations}
\label{fig:relationship_a}
\end{subfigure} \hspace*{0.1cm}
\begin{subfigure}{.23\textwidth}
\centering
\includegraphics[scale=0.63]{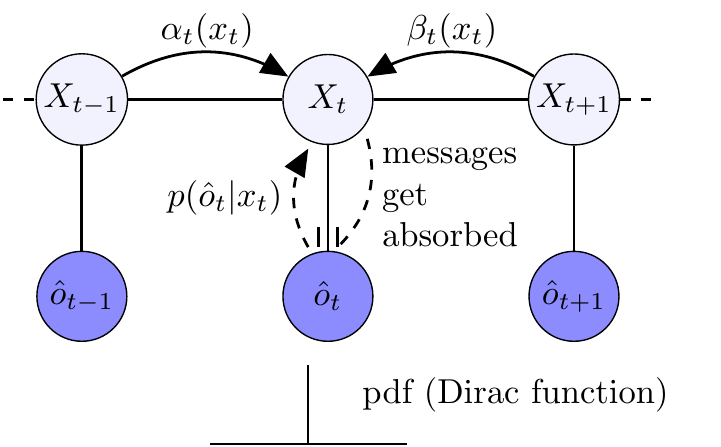}
\caption{Delta observations (standard)}
\label{fig:relationship_b}
\end{subfigure}
\caption{Relationship between standard and collective HMMs.}
\label{fig:relationship}
\end{figure}

\begin{rem}
Indeed, the relationship between standard HMM and collective HMM can be extended to more general graphical models. It turns out that for any arbitrary tree-structured graphical model, in the case of $M=1$ (fixed \textbf{delta} aggregate observations), the collective inference algorithm SBP coincides with the standard BP.
\end{rem}

\begin{figure*}[h]
    \centering
    \begin{subfigure}[b]{0.31\textwidth}
        \centering
        \includegraphics[width=1.0\textwidth]{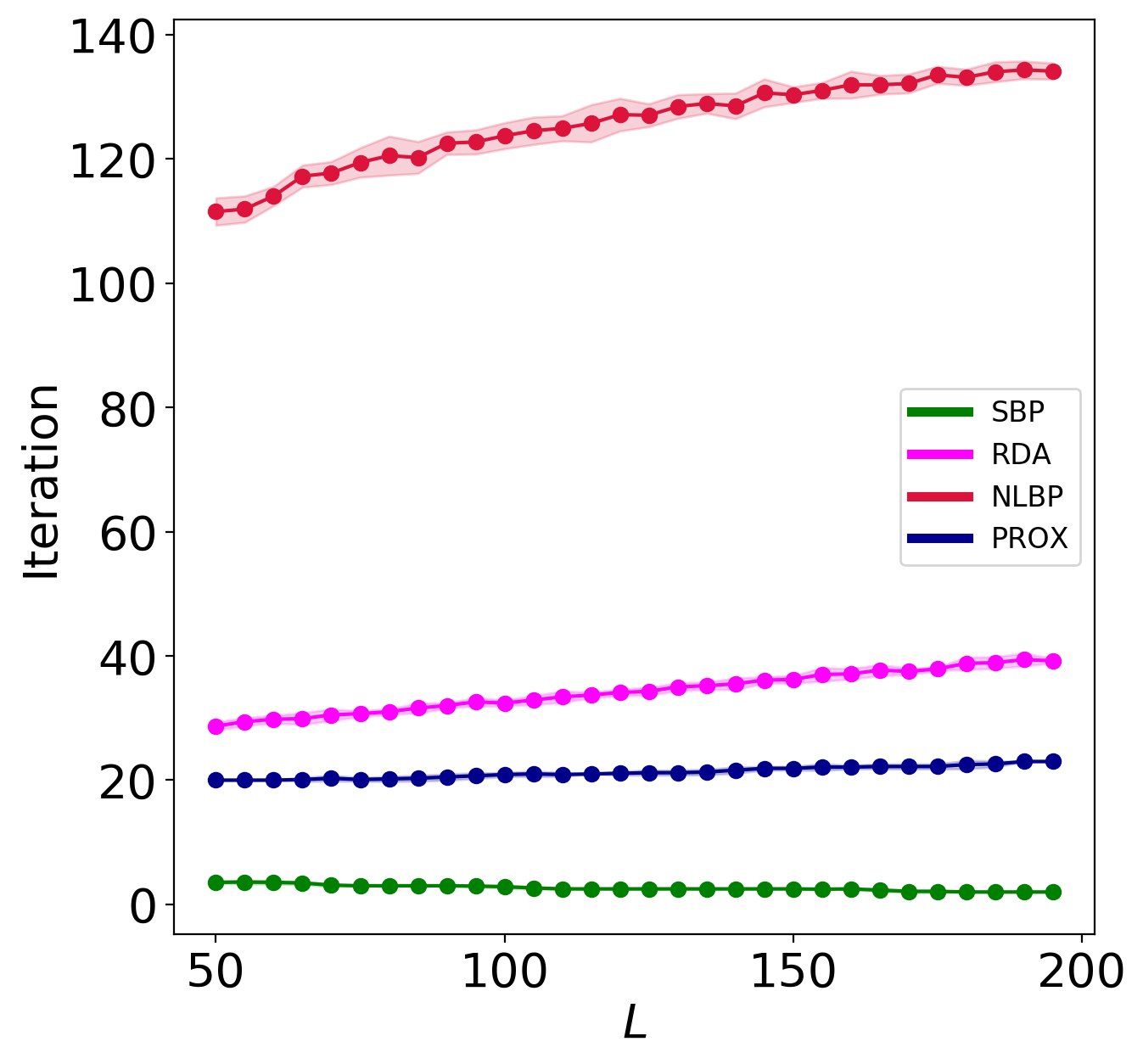}
        \caption{}
        \label{fig:vary_grid}
    \end{subfigure}
    \begin{subfigure}[b]{0.31\textwidth}
        \centering
        \includegraphics[width=1.0\textwidth]{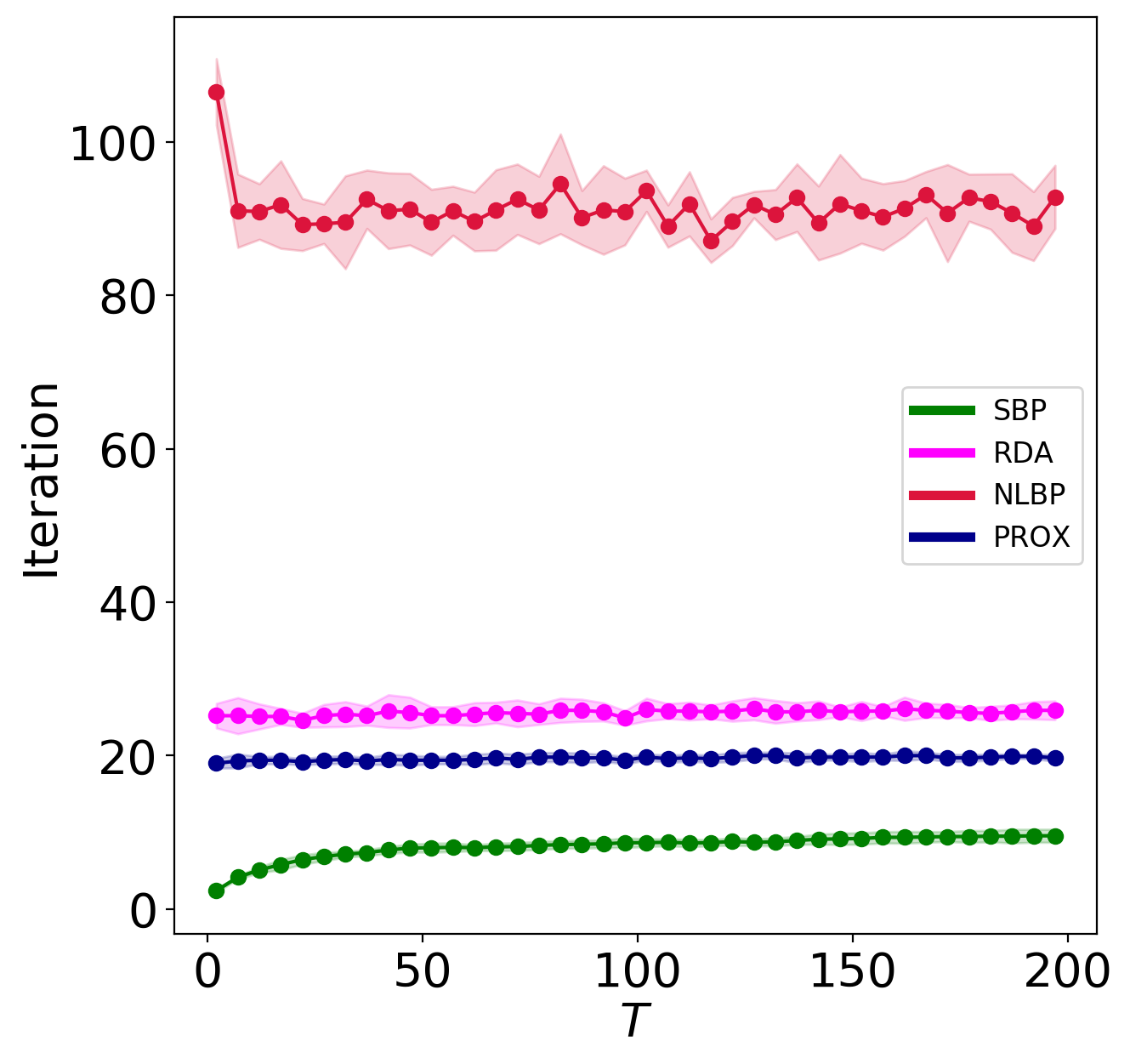}
        \caption{}
        \label{fig:vary_time}
    \end{subfigure}
    \begin{subfigure}[b]{0.31\textwidth}
        \centering
        \includegraphics[width=1.0\textwidth]{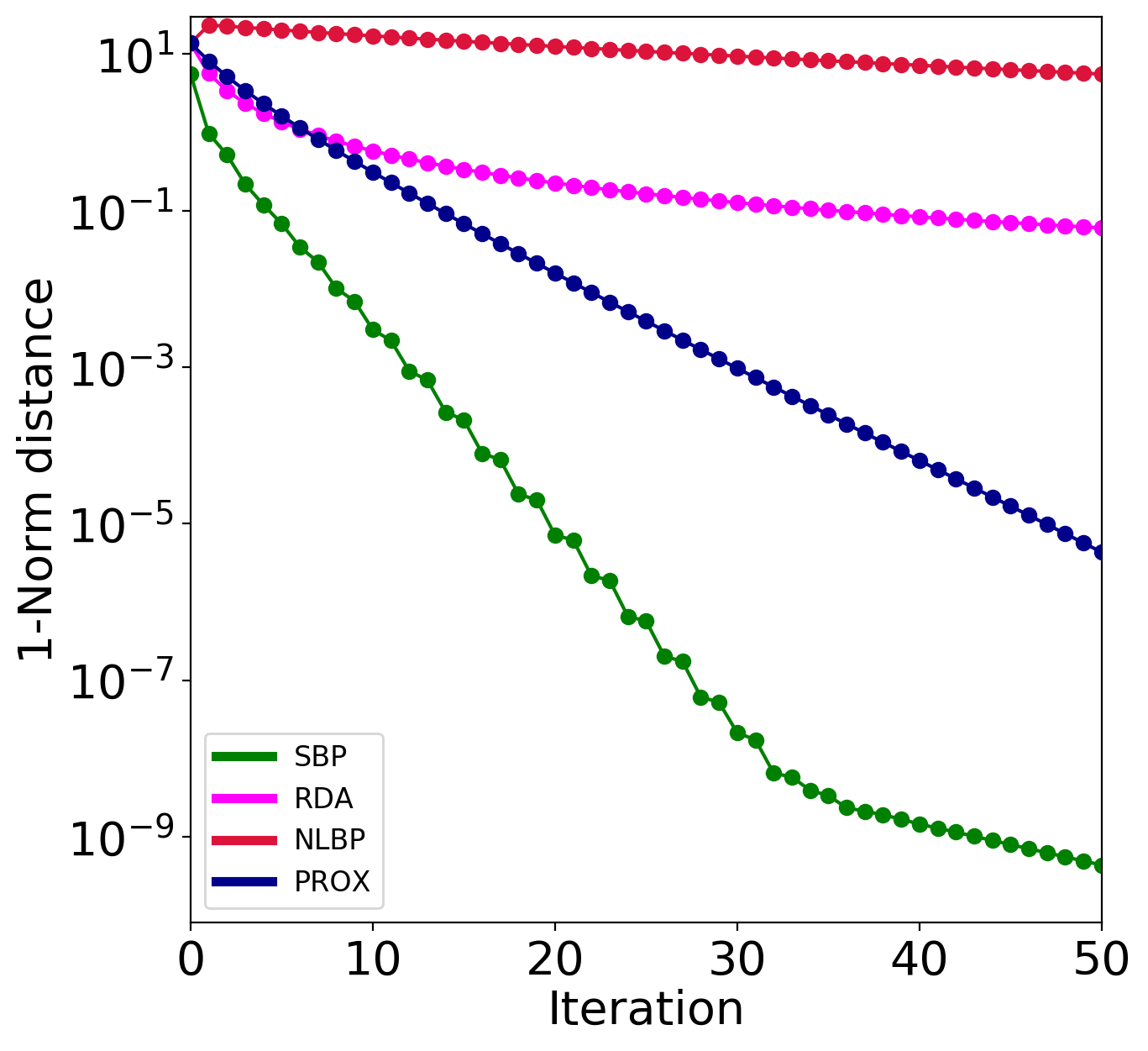}
        \caption{}
        \label{fig:convergence_time}
    \end{subfigure}
    \caption{Comparison of performance between NLBP, Bethe-RDA, PROX and SBP: (a) illustrates the performance for different grid sizes $L$ with a constant $T=20$, (b) compares the oracle complexity for different values of $T$ with fixed $L=50$, 
    (c) shows the convergent behavior in terms of 1-norm distance with respect to the optimal solutions.
    For both (a) and (b), each algorithm is run over $10$ different trials. The solid curves represent the mean time and the shaded regions represent the corresponding $\pm 1 \times$ standard deviation. For (c), the parameters are fixed as $L=50$ and $T=50$. }
    \label{fig:comparision}
\end{figure*}

\section{Evaluation}
\label{sec:eval}
We conduct three sets of experiments to evaluate the performance of our algorithms. The first one aims to evaluate the efficiency and convergence rate of SBP, compared with NLBP and Bethe-RDA. Note that since SBP uses a different observation model to NLBP and Bethe-RDA, this experiment is carried out only to compare the convergent behaviors of the algorithms. In the second experiment, we present an application of SBP in estimating ensembles with sparse information. We also empirically show that SBP algorithms can have good performance even in a PGM with loops. All the experiment were run on Intel
i7- 9700 CPU.
\subsection{Bird Migration}
\label{subsec:bird_migration}
First, we study a synthetic bird migration problem. Following the environment \cite{SunSheKum15}, we simulate $M$ birds flying over a $L \times L$ grid, aiming from bottom-left to top-right. The position transition probability between previous time and current time step follows a log-linear distribution that accounts for four factors: the distance between two positions, the angle between the movements direction and the wind, the angle between the direction of movement and the direction to the goal, and the preference to stay in the original cell. Each bird is simulated independently, following a $T$-step Markov chain. The parameter for the log-linear model is denoted by $\mathbf{w}$. In the NLBP setting, the sensors count the number of birds flying through each cell. Independent Poisson noise is added to each sensor measurement, which follows $y \propto {\rm Poisson}(\beta n)$. In the SBP setting, we use a Gaussian observation model. The probability for an individual bird to be observed by a sensor follows a Gaussian distribution centered at the sensor. In all experiments, we employ model parameters $\mathbf{w} = (3,5,5,10)$; the same parameters are used in the estimation algorithms. 
We set $M=5000$, and $\beta=1$. We compare the convergence performance between NLBP, Bethe-RDA and SBP with different $L$ and $T$ values, as depicted in Figure \ref{fig:comparision}.
In fact, we also implemented a simplified version, PROX, of Bethe-RDA which is based on standard proximal gradient algorithm with Kullback-Leibler divergence. When $T$ is fixed and $L$ varies, SBP is faster than NLBP, Bethe-RDA and PROX. When $L$ is fixed and $T$ varies, SBP is also faster than NLBP, Bethe-RDA and PROX. Moreover, we find that the run time does not grow much as $T$ increases. This makes SBP suitable for large HMMs. The convergence behaviors of NLBP, Bethe-RDA, PROX, and SBP are displayed in Figures \ref{fig:comparision}(c). Note that we show the total number of iterations instead of total running time for the comparison as all these algorithms have similar per iteration complexity. The 1-norm distance between true marginal distributions and the estimated distributions decrease monotonically for SBP, Bethe-RDA and PROX, whereas some instability occurs for NLBP. To stabilize NLBP, a damping ratio $\alpha$ is needed. However, higher damping rate implies smaller step size which in turn slows down the algorithm. We also find that the convergence property of NLBP is sensitive to the prior distribution and damping ratio when $T$ is large. 

\begin{figure*}[ht]
    \centering
    \begin{subfigure}[b]{0.48\textwidth}
        \centering
        \includegraphics[scale = 1.3]{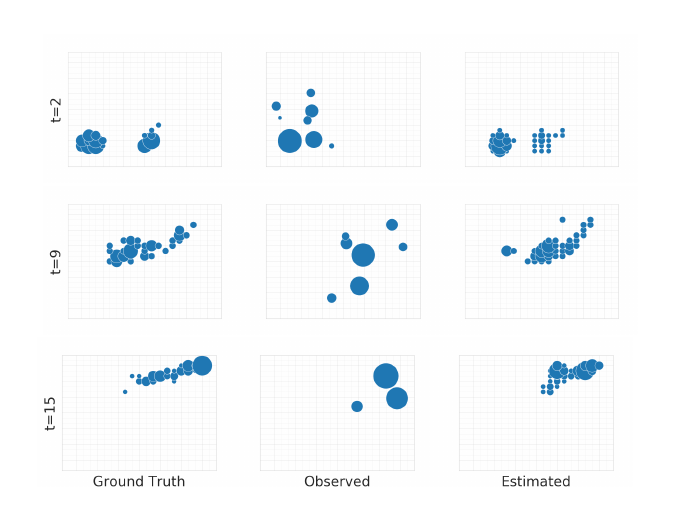}
        \caption{$M = 100$}
    \end{subfigure} \hspace{0.5cm}
    \begin{subfigure}[b]{0.48\textwidth}
        \centering
        \includegraphics[scale=1.3]{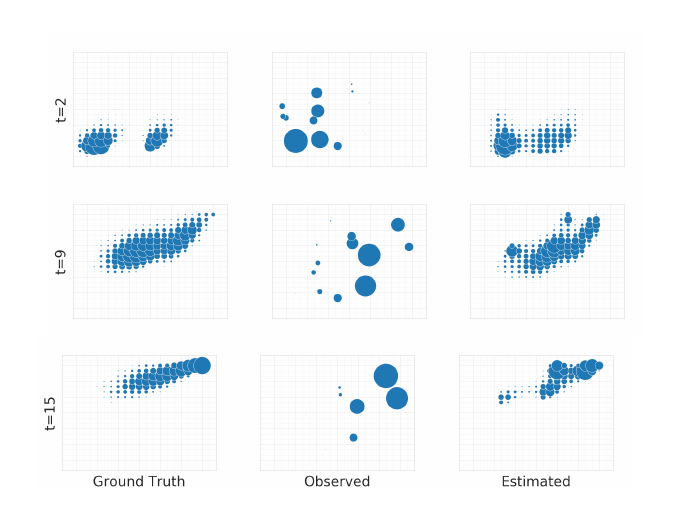}
        \caption{$M =10000$}
    \end{subfigure}
    \caption{Simulation of the movement of (a) 100 agents and (b) 10000 agents over $20 \times 20$ grid for $T=15$. In each of the figures, the first column depicts the real movement of agents at different time steps, the second column represents the aggregate sensor observations, and the third column depicts estimated aggregated positions. Here, the size of the circles is proportional to the number of agents.}
    \label{fig:ensemble_flow}
\end{figure*}

\begin{figure}[h]
    \centering
    \includegraphics[width=0.4\textwidth,height=0.32\textwidth]{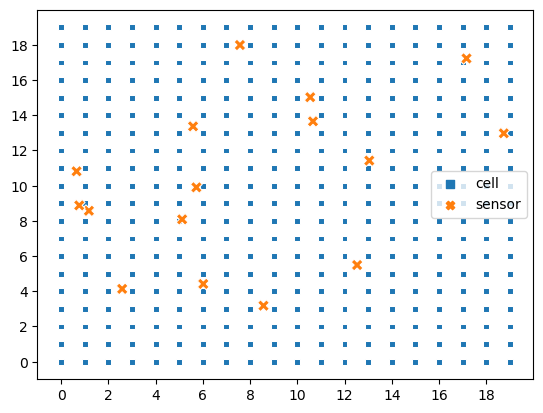}
    \caption{Sensor Location}
    \label{fig:my_label}
\end{figure}

\subsection{Ensemble Estimation with Sparse Information}
Next, we conduct a more challenging experiment wherein the aggregate observations are sparse. The task is to track ensembles over a network with limited number of sensors. The sensors can not tell the exact locations of the agents, such as in the case of Wi-Fi hotspots, or cell phone based stations, which can only tell the number of connected devices. Ensemble estimation is needed in tracking the human group activity without loss of individual privacy. We evaluate the model over a $20\times 20$ grid network with 16 sensors placed randomly as in Figure \ref{fig:my_label}. At each timestamp, each sensor observes a count, which records the number of agents connected to the sensor. Each agent can only connect with one sensor at a time and the probability of the connection decreases exponentially as the distance between agent and the sensor increases. To demonstrate the performance of SBP in estimating multi-modal distributions, we simulate a population with two clusters: one from left-bottom and one from center-bottom; both aim to the right-top corner of the grid in a $T=15$ time interval. We model the transition probability as in the bird migration setup discussed in Section~\ref{subsec:bird_migration}. 

We run simulations with 100 agents (Figure~\ref{fig:ensemble_flow}(a)) and 10000 agents (Figure~\ref{fig:ensemble_flow}(b)). As can be seen from the figures, even with such a sparse observation model, SBP can still infer the population movements to a satisfying accuracy.


\begin{figure}[ht]
    \centering
    \includegraphics[scale = 0.7]{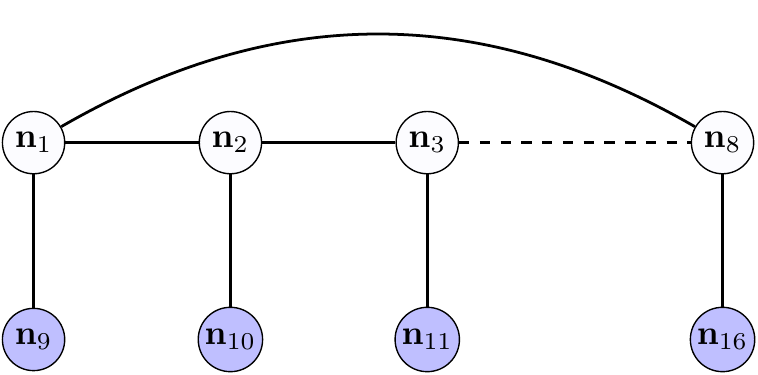}
    \caption{A loopy graph.}
    \label{fig:loopy}
\end{figure}

\subsection{Empirical Loopy Graph Validation}
Finally, we run a simple toy experiment where the underlying graph has loops as shown in Figure~\ref{fig:loopy}. We set $|\mathcal{X}_u| = |\mathcal{X}_o| = 5$ and generate the aggregate node distributions randomly. Moreover, the edge potential were generated using $\exp(I +Q)$, where $I$ is the identity matrix and $Q$ represents a random matrix generated by a Gaussian distribution. We estimate the marginal distribution by solving \eqref{eq:variational} using the SBP algorithm and then compare them with exact marginals by solving \eqref{eq:variational} using generic convex optimization algorithms. The convergence of the estimates for five different random seeds in terms or 1-norm distance is shown in Figure \ref{fig:16nodes}. We observe that the SBP algorithm has a good performance on the loopy graph.

\begin{figure}[ht]
    \centering
    \includegraphics[width = 0.7 \columnwidth]{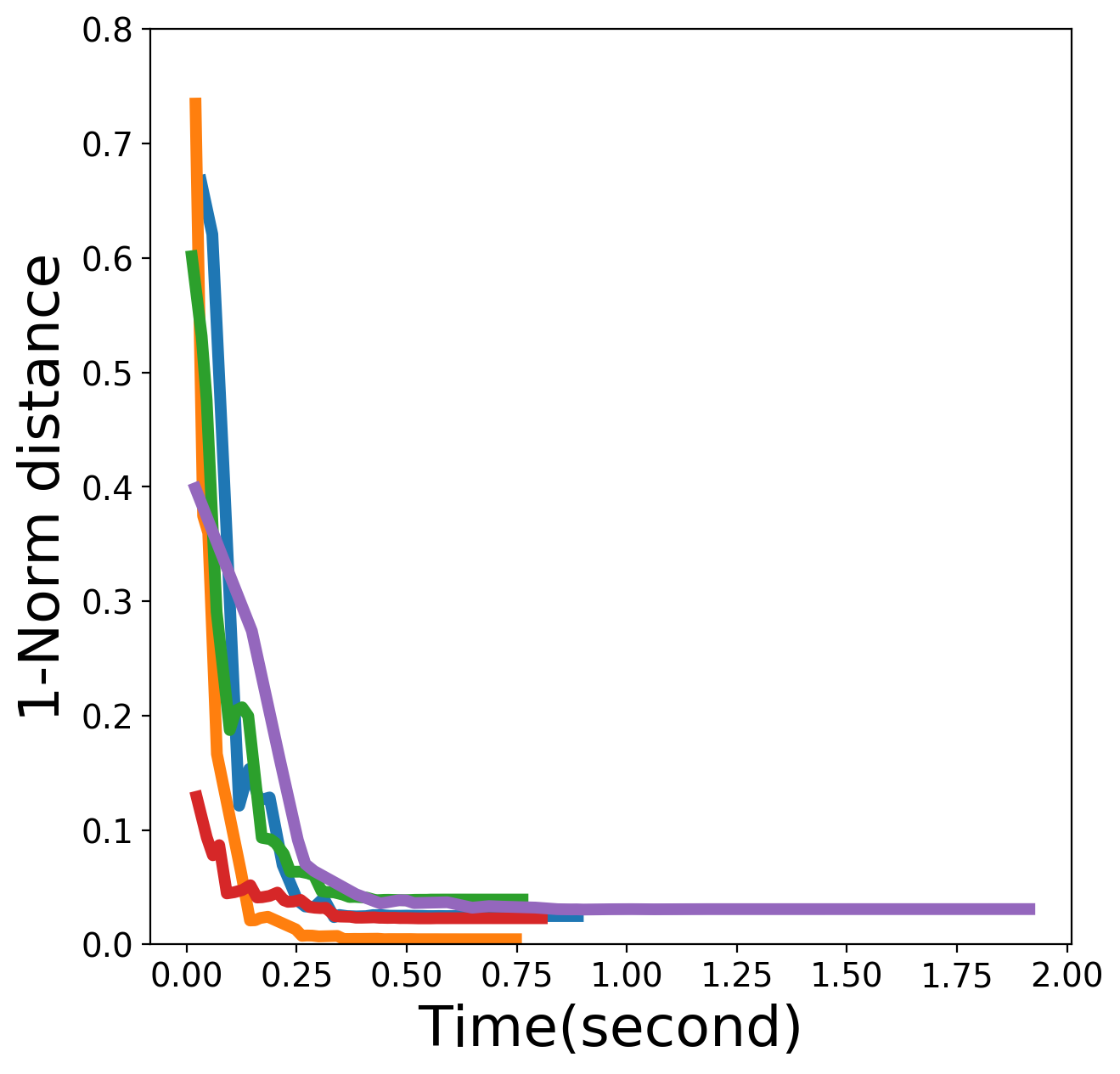} 
    \caption{Convergence of SBP on the loopy graph shown in Figure~\ref{fig:loopy}. Different colors represent different realizations.}
    \label{fig:16nodes}
\end{figure}
\setlength{\belowcaptionskip}{-10pt}
\section{Conclusion}
\label{sec:conclusion}

In this paper, we presented a reliable algorithm for inference/filtering from aggregate data based on multi-marginal optimal transport theory. We established that the aggregate inference/filtering problem is a special case of the entropic regularized MOT problem when the cost of MOT is structured according to the graphical model. We then combined the Sinkhorn algorithm for the MOT problems and the standard belief propagation algorithm to establish our method. Our algorithm enjoys fast convergence and has a convergence guarantee when the underlying graph structure is a tree. For the cases of HMMs which are widely used in control and estimation, we specialize our SBP algorithm to establish the collective forward-backward algorithm. The latter naturally generalizes the forward-backward algorithm in standard filtering problems for HMMs. We evaluated the performance of our algorithm on multiple applications involving inference from aggregate data such as bird migration and human mobility based on hidden Markov models. In the future, we plan to extend our current algorithm to cover graphical models with continuous state space.
We also plan to investigate the parameter learning of CGMs using the MOT framework.



\bibliographystyle{IEEEtran}
\bibliography{./refs}

\begin{thebibliography}{10}
\providecommand{\url}[1]{#1}
\csname url@rmstyle\endcsname
\providecommand{\newblock}{\relax}
\providecommand{\bibinfo}[2]{#2}
\providecommand\BIBentrySTDinterwordspacing{\spaceskip=0pt\relax}
\providecommand\BIBentryALTinterwordstretchfactor{4}
\providecommand\BIBentryALTinterwordspacing{\spaceskip=\fontdimen2\font plus
\BIBentryALTinterwordstretchfactor\fontdimen3\font minus
  \fontdimen4\font\relax}
\providecommand\BIBforeignlanguage[2]{{%
\expandafter\ifx\csname l@#1\endcsname\relax
\typeout{** WARNING: IEEEtran.bst: No hyphenation pattern has been}%
\typeout{** loaded for the language `#1'. Using the pattern for}%
\typeout{** the default language instead.}%
\else
\language=\csname l@#1\endcsname
\fi
#2}}

\bibitem{WaiJor08}
M.~J. Wainwright and M.~I. Jordan, ``Graphical models, exponential families,
  and variational inference,'' \emph{Foundations and Trends{\textregistered} in
  Machine Learning}, vol.~1, no. 1--2, pp. 1--305, 2008.

\bibitem{Thr02}
S.~Thrun, ``Probabilistic robotics,'' \emph{Communications of the ACM},
  vol.~45, no.~3, pp. 52--57, 2002.

\bibitem{AndMoo12}
B.~D. Anderson and J.~B. Moore, \emph{Optimal filtering}.\hskip 1em plus 0.5em
  minus 0.4em\relax Courier Corporation, 2012.

\bibitem{KolFri09}
D.~Koller and N.~Friedman, \emph{Probabilistic graphical models: principles and
  techniques}.\hskip 1em plus 0.5em minus 0.4em\relax MIT press, 2009.

\bibitem{BoxTia11}
G.~E. Box and G.~C. Tiao, \emph{Bayesian inference in statistical
  analysis}.\hskip 1em plus 0.5em minus 0.4em\relax John Wiley \& Sons, 2011,
  vol.~40.

\bibitem{Pea88}
J.~Pearl, ``Probabilistic reasoning in intelligent systems: Networks of
  plausible inference,'' \emph{Morgan Kaufmann Publishers Inc}, 1988.

\bibitem{YedFreWei01}
J.~S. Yedidia, W.~T. Freeman, and Y.~Weiss, ``Generalized belief propagation,''
  in \emph{Advances in neural information processing systems}, 2001, pp.
  689--695.

\bibitem{YedFreWei03}
J.~S. Yedidia, W.~T. Freeman, and Y.~Weiss, ``Understanding belief propagation
  and its generalizations,'' \emph{Exploring artificial intelligence in the new
  millennium}, vol.~8, pp. 236--239, 2003.

\bibitem{KscFreLoe01}
F.~R. Kschischang, B.~J. Frey, and H.-A. Loeliger, ``Factor graphs and the
  sum-product algorithm,'' \emph{IEEE Transactions on information theory},
  vol.~47, no.~2, pp. 498--519, 2001.

\bibitem{Sun75}
R.~Sundberg, ``Some results about decomposable (or {M}arkov-type) models for
  multidimensional contingency tables: distribution of marginals and
  partitioning of tests,'' \emph{Scandinavian Journal of Statistics}, pp.
  71--79, 1975.

\bibitem{Mac77}
E.~C. MacRae, ``Estimation of time-varying {M}arkov processes with aggregate
  data,'' \emph{Econometrica: journal of the Econometric Society}, pp.
  183--198, 1977.

\bibitem{KalLawVol83}
J.~D. Kalbfleisch, J.~F. Lawless, and W.~M. Vollmer, ``Estimation in {M}arkov
  models from aggregate data,'' \emph{Biometrics}, pp. 907--919, 1983.

\bibitem{SheDie11}
D.~R. Sheldon and T.~G. Dietterich, ``Collective graphical models,'' in
  \emph{Advances in Neural Information Processing Systems}, 2011, pp.
  1161--1169.

\bibitem{LuoXuZhe16}
D.~Luo, H.~Xu, Y.~Zhen, B.~Dilkina, H.~Zha, X.~Yang, and W.~Zhang, ``Learning
  mixtures of {M}arkov chains from aggregate data with structural
  constraints,'' \emph{IEEE Transactions on Knowledge and Data Engineering},
  vol.~28, no.~6, pp. 1518--1531, 2016.

\bibitem{SheSunKumDie13}
D.~Sheldon, T.~Sun, A.~Kumar, and T.~Dietterich, ``Approximate inference in
  collective graphical models,'' in \emph{International Conference on Machine
  Learning}, 2013, pp. 1004--1012.

\bibitem{SunSheKum15}
T.~Sun, D.~Sheldon, and A.~Kumar, ``Message passing for collective graphical
  models,'' in \emph{International Conference on Machine Learning}, 2015, pp.
  853--861.

\bibitem{VilBelSheMcc15}
L.~Vilnis, D.~Belanger, D.~Sheldon, and A.~McCallum, ``Bethe projections for
  non-local inference,'' in \emph{Proceedings of the Thirty-First Conference on
  Uncertainty in Artificial Intelligence}, 2015, pp. 892--901.

\bibitem{GanSwi98}
W.~Gangbo and A.~{\'S}wiech, ``Optimal maps for the multidimensional
  {M}onge-{K}antorovich problem,'' \emph{Communications on Pure and Applied
  Mathematics: A Journal Issued by the Courant Institute of Mathematical
  Sciences}, vol.~51, no.~1, pp. 23--45, 1998.

\bibitem{Pas15}
B.~Pass, ``Multi-marginal optimal transport: theory and applications,''
  \emph{ESAIM: Mathematical Modelling and Numerical Analysis}, vol.~49, no.~6,
  pp. 1771--1790, 2015.

\bibitem{BenCarCut15}
J.-D. Benamou, G.~Carlier, M.~Cuturi, L.~Nenna, and G.~Peyr{\'e}, ``Iterative
  {B}regman projections for regularized transportation problems,'' \emph{SIAM
  Journal on Scientific Computing}, vol.~37, no.~2, pp. A1111--A1138, 2015.

\bibitem{Nen16}
L.~Nenna, ``Numerical methods for multi-marginal optimal transportation,''
  Ph.D. dissertation, 2016.

\bibitem{Pas12}
B.~Pass, ``On the local structure of optimal measures in the multi-marginal
  optimal transportation problem,'' \emph{Calculus of Variations and Partial
  Differential Equations}, vol.~43, no. 3-4, pp. 529--536, 2012.

\bibitem{Vil03}
C.~Villani, \emph{Topics in optimal transportation}.\hskip 1em plus 0.5em minus
  0.4em\relax American Mathematical Soc., 2003, no.~58.

\bibitem{Sin64}
R.~Sinkhorn, ``A relationship between arbitrary positive matrices and doubly
  stochastic matrices,'' \emph{The annals of mathematical statistics}, vol.~35,
  no.~2, pp. 876--879, 1964.

\bibitem{FraLor89}
J.~Franklin and J.~Lorenz, ``On the scaling of multidimensional matrices,''
  \emph{Linear Algebra and its applications}, vol. 114, pp. 717--735, 1989.

\bibitem{Cut13}
M.~Cuturi, ``Sinkhorn distances: Lightspeed computation of optimal transport,''
  in \emph{Advances in neural information processing systems}, 2013, pp.
  2292--2300.

\bibitem{PasFuBou12}
A.~Pasanisi, S.~Fu, and N.~Bousquet, ``Estimating discrete {M}arkov models from
  various incomplete data schemes,'' \emph{Computational Statistics \& Data
  Analysis}, vol.~56, no.~9, pp. 2609--2625, 2012.

\bibitem{BerShe16}
G.~Bernstein and D.~Sheldon, ``Consistently estimating {M}arkov chains with
  noisy aggregate data,'' in \emph{Artificial Intelligence and Statistics},
  2016, pp. 1142--1150.

\bibitem{CheKar18}
Y.~Chen and J.~Karlsson, ``State tracking of linear ensembles via optimal mass
  transport,'' \emph{IEEE Control Systems Letters}, vol.~2, no.~2, pp.
  260--265, 2018.

\bibitem{HasRinChe19}
I.~Haasler, A.~Ringh, Y.~Chen, and J.~Karlsson, ``Estimating ensemble flows on
  a hidden {M}arkov chain,'' in \emph{58th IEEE Conference on Decision and
  Control}, 2019.

\bibitem{CheGeoPav15a}
Y.~Chen, T.~T. Georgiou, and M.~Pavon, ``Optimal steering of a linear
  stochastic system to a final probability distribution, part {I},'' \emph{IEEE
  Transactions on Automatic Control}, vol.~61, no.~5, pp. 1158--1169, 2015.

\bibitem{TagMeh20}
A.~Taghvaei and P.~G. Mehta, ``An optimal transport formulation of the ensemble
  {K}alman filter,'' \emph{IEEE Transactions on Automatic Control}, 2020.

\bibitem{Leo14}
C.~L{\'e}onard, ``A survey of the {S}chr\"odinger problem and some of its
  connections with optimal transport,'' \emph{DYNAMICAL SYSTEMS}, vol.~34,
  no.~4, pp. 1533--1574, 2014.

\bibitem{CheGeoPav14e}
Y.~Chen, T.~T. Georgiou, and M.~Pavon, ``On the relation between optimal
  transport and {S}chr{\"o}dinger bridges: A stochastic control viewpoint,''
  \emph{Journal of Optimization Theory and Applications}, vol. 169, no.~2, pp.
  671--691, 2016.

\bibitem{CheConGeoRip19}
Y.~Chen, G.~Conforti, T.~T. Georgiou, and L.~Ripani, ``Multi-marginal
  {S}chr{\"o}dinger bridges,'' in \emph{International Conference on Geometric
  Science of Information}.\hskip 1em plus 0.5em minus 0.4em\relax Springer,
  2019, pp. 725--732.

\bibitem{Kal60}
R.~E. Kalman, ``A new approach to linear filtering and prediction problems,''
  1960.

\bibitem{Rab89}
L.~R. Rabiner, ``A tutorial on hidden markov models and selected applications
  in speech recognition,'' \emph{Proceedings of the IEEE}, vol.~77, no.~2, pp.
  257--286, 1989.

\bibitem{ItoXio00}
K.~Ito and K.~Xiong, ``Gaussian filters for nonlinear filtering problems,''
  \emph{IEEE transactions on automatic control}, vol.~45, no.~5, pp. 910--927,
  2000.

\bibitem{YanMehMey13}
T.~Yang, P.~G. Mehta, and S.~P. Meyn, ``Feedback particle filter,'' \emph{IEEE
  transactions on Automatic control}, vol.~58, no.~10, pp. 2465--2480, 2013.

\bibitem{LorNvd11}
R.~J. Lorentzen and G.~N{\ae}vdal, ``An iterative ensemble {K}alman filter,''
  \emph{IEEE Transactions on Automatic Control}, vol.~56, no.~8, pp.
  1990--1995, 2011.

\bibitem{YedFreWei05}
J.~S. Yedidia, W.~T. Freeman, and Y.~Weiss, ``Constructing free-energy
  approximations and generalized belief propagation algorithms,'' \emph{IEEE
  Transactions on information theory}, vol.~51, no.~7, pp. 2282--2312, 2005.

\bibitem{Xia09}
L.~Xiao, ``Dual averaging method for regularized stochastic learning and online
  optimization,'' in \emph{Advances in Neural Information Processing Systems},
  2009, pp. 2116--2124.

\bibitem{ElvHaaJakKar20}
F.~Elvander, I.~Haasler, A.~Jakobsson, and J.~Karlsson, ``Multi-marginal
  optimal transport using partial information with applications in robust
  localization and sensor fusion,'' \emph{Signal Processing}, 2020.

\bibitem{CheGeoPav16}
Y.~Chen, T.~Georgiou, and M.~Pavon, ``Entropic and displacement interpolation:
  a computational approach using the {H}ilbert metric,'' \emph{SIAM Journal on
  Applied Mathematics}, vol.~76, no.~6, pp. 2375--2396, 2016.

\bibitem{DemSte40}
W.~E. Deming and F.~F. Stephan, ``On a least squares adjustment of a sampled
  frequency table when the expected marginal totals are known,'' \emph{The
  Annals of Mathematical Statistics}, vol.~11, no.~4, pp. 427--444, 1940.

\bibitem{LuoTse92}
Z.-Q. Luo and P.~Tseng, ``On the convergence of the coordinate descent method
  for convex differentiable minimization,'' \emph{Journal of Optimization
  Theory and Applications}, vol.~72, no.~1, pp. 7--35, 1992.

\bibitem{HaaRinCheKar20}
I.~Haasler, A.~Ringh, Y.~Chen, and J.~Karlsson, ``Multi-marginal optimal
  transport and {S}chr\"odinger bridges on trees,'' \emph{arXiv preprint
  arXiv:2004.06909}, 2020.

\bibitem{FinSinTis98}
S.~Fine, Y.~Singer, and N.~Tishby, ``The hierarchical hidden {M}arkov model:
  Analysis and applications,'' \emph{Machine learning}, vol.~32, no.~1, pp.
  41--62, 1998.

\bibitem{TehWel02}
Y.~W. Teh and M.~Welling, ``The unified propagation and scaling algorithm,'' in
  \emph{Advances in neural information processing systems}, 2002, pp. 953--960.

\bibitem{Var84}
S.~S. Varadhan, \emph{Large deviations and applications}.\hskip 1em plus 0.5em
  minus 0.4em\relax Society for Industrial and Applied Mathematics(SIAM), 1984.

\bibitem{BoyVan04}
S.~Boyd and L.~Vandenberghe, \emph{Convex optimization}.\hskip 1em plus 0.5em
  minus 0.4em\relax Cambridge university press, 2004.

\bibitem{LuoTse93}
Z.-Q. Luo and P.~Tseng, ``On the convergence rate of dual ascent methods for
  linearly constrained convex minimization,'' \emph{Mathematics of Operations
  Research}, vol.~18, no.~4, pp. 846--867, 1993.

\end{thebibliography}

\newpage 

\appendix

\subsection{Proof of Theorem~\ref{thm:is_bp}}
\label{appendix:proof_thm1}
We first construct a Lagrangian for Problem~\eqref{eq:MOT_bethe} with Lagrange multipliers $\lambda_i(x_i)$ corresponding to the fixed marginal constraints \eqref{eq:MOT_bethe_b}, $\nu_{ji}(x_i)$ corresponding to the marginalization constraints \eqref{eq:MOT_bethe_c}, and $\zeta_i$ corresponding to the normalization constraints \eqref{eq:MOT_bethe_d}. 
This yields the Lagrangian
\begin{align}\label{eq:Lagrangian}
    \mathcal{L} =&~ F_{\rm Bethe}(\bn) + \sum_{i \in \Gamma} \sum_{x_i} \lambda_i(x_i) \left( n_i(x_i) - y_i(x_i) \right) \nonumber \\
    &~ + \sum_{i\in V} \sum_{j \in N(i)} \sum_{x_i} \nu_{ji}(x_i) \left( \sum_{x_j} n_{ij}(x_i,x_j) - n_i(x_i) \right)  \nonumber \\
    &~ + \sum_{i \in V} \zeta_i \left( \sum_{x_i} n_i(x_i) - 1 \right).
\end{align}
Differentiating the Lagrangian with respect to the aggregate marginals $\bn_{ij}$ and equating this derivative to zero, we obtain
\begin{align}\label{eq:proof_n_ij}
    \frac{\partial \cL}{\partial n_{ij}(x_i,x_j)} \!&=\! 1 + \ln~\frac{n_{ij}(x_i,x_j)}{\psi_{ij}(x_i,x_j)} +  \nu_{ji}(x_i) +  \nu_{ij}(x_j) = 0 \nonumber \\
    \Rightarrow  n_{ij} (x_i,x_j) &\propto \psi_{ij}(x_i,x_j) \mathrm{exp} \left(  -   \nu_{ji}(x_i) -  \nu_{ij}(x_j) \right).
\end{align}
Now differentiating the Lagrangian with respect to $\bn_i$, for $i \notin \Gamma$ and $d_i > 1$, we have
\begin{align} \label{eq:lagrange_ni}
    \frac{\partial \cL}{\partial n_{i}(x_i)} &= \zeta_i- (d_i -1) [1 + \ln n_i(x_i)] - \sum_{j \in N(i)} \nu_{ji}(x_i)= 0 \nonumber \\
    \Rightarrow ~ n_{i} (x_i) &\propto \mathrm{exp} \left( - \frac{1}{d_i - 1} \left\{  \sum_{j \in N(i)} \nu_{ji}(x_i) \right\} \right).
\end{align} 
Similarly, when $i \notin \Gamma$ and $d_i = 1$, we have
\begin{align} \label{eq:lagrange_ni_one}
    \frac{\partial ~\cL}{\partial ~n_{i}(x_i)} = 0 ~\Rightarrow~ \nu_{ji}(x_i) - \zeta_i = 0,
\end{align} 
where $j\in N(i)$ is the only neighbor of node $i$.

Denote
\begin{equation}\label{eq:denote_message}
    m_{i \rightarrow j}(x_j) := \sum_{x_i} \psi_{ij}(x_i,x_j)~ \mathrm{exp}(- \nu_{ji}(x_i)).
\end{equation}
Using \eqref{eq:proof_n_ij}, \eqref{eq:lagrange_ni} and marginalization constraint \eqref{eq:MOT_bethe_c}, we obtain, for $i$ such that $d_i>1$,
\begin{align}\label{eq:marg_n_i}
    &\sum_{x_j} \psi_{ij}(x_i,x_j)~ \mathrm{exp} \left( -   \nu_{ji}(x_i) -  \nu_{ij}(x_j)  \right) \propto  \nonumber \\
     &\quad \quad  \mathrm{exp} \left( - \frac{1}{d_i - 1} \left\{  \sum_{j \in N(i)} \nu_{ji}(x_i) \right\} \right) \nonumber \\
    &\Rightarrow m_{j \rightarrow i} (x_i)~ \mathrm{exp}(-\nu_{ji}(x_i)) \propto \prod_{j \in N(i)} \mathrm{exp}~(- \nu_{ji}(x_i))^{\frac{1}{d_i - 1}} \nonumber \\
    &\Rightarrow \prod_{j \in N(i)} m_{j \rightarrow i} (x_i)~ \mathrm{exp}(-\nu_{ji}(x_i)) \propto \nonumber \\
     &\quad \prod_{j \in N(i)} \left( \mathrm{exp}~(- \nu_{ji}(x_i))^{\frac{1}{d_i - 1}} \prod_{k \in N(i) \backslash j} \mathrm{exp}~(- \nu_{ki}(x_i))^{\frac{1}{d_i - 1}} \right).  \nonumber
     \end{align}
    It follows that, by leveraging the fact that $|N(i)|=d_i$,
    \begin{align}
     &\prod_{j \in N(i)} m_{j \rightarrow i} (x_i)~ \mathrm{exp}(-\nu_{ji}(x_i)) \propto \nonumber \\
     &\quad \left( \prod_{j \in N(i)} \mathrm{exp}~(- \nu_{ji}(x_i))^{\frac{1}{d_i - 1}} \right) \left( \prod_{k \in N(i)} \mathrm{exp}~(- \nu_{ki}(x_i)) \right)  \nonumber \\
     &\Rightarrow \prod_{j \in N(i)} m_{j \rightarrow i} (x_i) \propto \prod_{j \in N(i)} \mathrm{exp}~(-\nu_{ji}(x_i)^{\frac{1}{d_i - 1}} .
\end{align}
Combining \eqref{eq:lagrange_ni} and \eqref{eq:marg_n_i}, we arrive at
\begin{align}\label{eq:proved_ni}
    n_i(x_i) \propto \prod_{j \in N(i)} m_{j \rightarrow i} (x_i),
\end{align}
which is \eqref{eq:nmarginal} for $d_i > 1$ and $i \notin \Gamma$. 

For  $i \notin \Gamma$ and $d_i = 1$, i.e., unconstrained variable at a leaf node, using \eqref{eq:proof_n_ij} and \eqref{eq:lagrange_ni_one} we deduce that \begin{align}
    n_{ij}(x_i,x_j) \propto  \psi_{ij}(x_i,x_j)~ \mathrm{exp} \left(-  \nu_{ij}(x_j)  \right).
\end{align}
In view of \eqref{eq:denote_message} and marginalization constraint \eqref{eq:MOT_bethe_c}, we get
\begin{align}
    n_i(x_i) &= \sum_{x_j}  n_{ij}(x_i,x_j) \nonumber \\
     &\propto \sum_{x_j} \psi_{ij}(x_i,x_j)~ \mathrm{exp} \left(-  \nu_{ij}(x_j)  \right) \nonumber \\
     &= m_{j \rightarrow i}(x_i),
\end{align}
which is \eqref{eq:nmarginal} for $d_i =1$ and $i \notin \Gamma$.

Combining \eqref{eq:proof_n_ij} and \eqref{eq:MOT_bethe_c}, we obtain
\begin{align}\label{eq:proved_message}
    &\sum_{x_j} \psi_{ij}(x_i,x_j)~ \mathrm{exp} \left(  -   \nu_{ji}(x_i) -  \nu_{ij}(x_j)  \right) \propto n_i(x_i) \nonumber \\
    &\Rightarrow m_{j \rightarrow i} (x_i)~ \mathrm{exp}(-\nu_{ji}(x_i)) \propto n_i(x_i) \nonumber \\
    &\Rightarrow \mathrm{exp}(-\nu_{ji}(x_i)) \propto \frac{n_i(x_i) }{m_{j \rightarrow i} (x_i)} \nonumber \\
    &\Rightarrow \sum_{x_i} \psi_{ij}(x_i,x_j) \mathrm{exp}(-\nu_{ji}(x_i)) \propto \sum_{x_i} \psi_{ij}(x_i,x_j) \frac{n_i(x_i) }{m_{j \rightarrow i} (x_i)} \nonumber \\
    &\Rightarrow m_{i \rightarrow j}(x_j) \propto \sum_{x_i} \psi_{ij}(x_i,x_j) \frac{n_i(x_i) }{m_{j \rightarrow i} (x_i)}.
\end{align}

Therefore, for $i \notin \Gamma$, using \eqref{eq:proved_ni} and \eqref{eq:proved_message}, we get
\begin{align}
    m_{i \rightarrow j}(x_j) &\propto \sum_{x_i} \psi_{ij}(x_i,x_j) \frac{\prod_{k \in N(i)} m_{k \rightarrow i} (x_i) }{m_{j \rightarrow i} (x_i)} \nonumber \\
    &\propto \sum_{x_i} \psi_{ij}(x_i,x_j) \prod_{k \in N(i) \backslash j} m_{k \rightarrow i} (x_i),
\end{align}
which is \eqref{eq:is_bp_MOT1}.

Furthermore, for $i \in \Gamma$, combining \eqref{eq:proved_message} and observation constraint \eqref{eq:MOT_bethe_b}, we arrive at
\begin{align}
    m_{i \rightarrow j}(x_j) \propto \sum_{x_i} \psi_{ij}(x_i,x_j) \frac{y_i(x_i) }{m_{j \rightarrow i} (x_i)},
\end{align}
which is \eqref{eq:is_bp_MOT2}. This completes the proof.

\end{document}